%% file: main.tex
\documentclass[hidelinks,11pt]{article} 
\usepackage{fullpage}
\usepackage{microtype}
\usepackage{graphicx}
\usepackage{subfigure}
\usepackage{booktabs} 
\usepackage{hyperref}
\usepackage{xcolor}
\hypersetup{
	colorlinks,
	linkcolor={black},
	citecolor={blue!50!black}
}
\usepackage{algorithmic}
\usepackage{algorithm}

\usepackage[utf8]{inputenc} 
\usepackage[T1]{fontenc}    
\usepackage{url}            
\usepackage{amsfonts}       
\usepackage{nicefrac}       
\usepackage{amsmath,amssymb,amsthm}
\usepackage{color}
\usepackage{mathtools}
\usepackage{enumitem}
\usepackage{epstopdf}
\usepackage{tabulary}
\usepackage{natbib}
\usepackage{chngcntr}

\newtheorem{theorem}{Theorem}
\newtheorem{lemma}{Lemma}
\newtheorem{proposition}{Proposition}
\newtheorem{problem}{Problem}
\newtheorem{assumption}{Assumption}

\newcommand{\paren}[1]{\ensuremath{\left( #1\right)}}

\newcommand{\clint}[1]{\ensuremath{\left[ #1\right]}}
\newcommand{\set}[1]{\ensuremath{\left\{ #1\right\}}}
\newcommand{\matr}[1]{\ensuremath{\clint{\begin{array} #1 \end{array}}}}
\newcommand{\norm}[1]{\ensuremath{\left\| #1\right\|}}
\newcommand{\snorm}[1]{\ensuremath{\| #1\|}}
\newcommand{\abs}[1]{\ensuremath{\left| #1\right|}}

\newcommand{\K}{\ensuremath{\mathcal{K}}}
\newcommand{\N}{\ensuremath{\mathcal{N}}}
\newcommand{\R}{\ensuremath{\mathbb{R}}}
\renewcommand{\P}{\ensuremath{\mathbb{P}}}
\renewcommand{\S}{\ensuremath{\mathbb{S}}}
\newcommand{\Reg}{\ensuremath{\mathcal{R}}}
\newcommand{\F}{\ensuremath{\mathcal{F}}}
\newcommand{\E}{\ensuremath{\mathbb{E}}}

\renewcommand{\O}{\ensuremath{\mathcal{O}}}
\newcommand{\C}{\ensuremath{\mathcal{C}}}
\newcommand{\T}{\ensuremath{\mathcal{T}}}

\newcommand{\LL}{\ensuremath{\mathcal{L}}}
\DeclareMathOperator{\Tr}{\mathrm{tr}}

\DeclarePairedDelimiter{\diagfences}{(}{)}
\newcommand{\diag}{\operatorname{diag}\diagfences}

\title{Online Learning of the Kalman Filter with Logarithmic Regret}
\author{Anastasios~Tsiamis $^\star$ and George~J.~Pappas 
	\thanks{The authors  are   with   the   Department   of   Electrical   and   Systems  Engineering,  University  of  Pennsylvania,  Philadelphia,  PA  19104.
		Emails: \{atsiamis,pappasg\}@seas.upenn.edu}
}
\begin{document}
\maketitle

\begin{abstract}
In this paper, we consider the problem of predicting observations generated online by an unknown, partially observed linear system, which is driven by stochastic noise. For such systems the optimal predictor in the mean square sense is the celebrated Kalman filter, which can be explicitly computed when the system model is known. When the system model is unknown, we have to learn how to predict observations online based on finite data, suffering possibly a non-zero regret with respect to the Kalman filter's prediction.
We show that it is possible to achieve a regret of the order of $\mathrm{poly}\log(N)$ with high probability, where $N$ is the number of observations collected. 
Our work is the first to provide logarithmic regret guarantees for the widely used Kalman filter. 
This is achieved using an online least-squares  algorithm, which exploits the approximately linear relation between future observations and past observations. The regret analysis is based on the stability properties of the Kalman filter, recent statistical tools for finite sample analysis of system identification, and classical results for the analysis of least-squares algorithms for time series. 
Our regret analysis can also be applied for state prediction of the hidden state, in the case of unknown noise statistics but known state-space basis.
A fundamental technical contribution is that our bounds hold even for the class of non-explosive systems, which includes the class of marginally stable systems, which was an open problem for the case of online prediction under stochastic noise.
\end{abstract}

\input{Introduction}
\input{Formulation}
\input{Output_Prediction}
\input{Regret_Analysis}

\input{Extensions}
\input{Conclusion}
\bibliography{Regret}
\bibliographystyle{plainnat}

\appendix
\counterwithin{lemma}{section}
\counterwithin{theorem}{section}
\counterwithin{proposition}{section}
\counterwithin{corollary}{section}
\counterwithin{definition}{section}
\counterwithin{equation}{section}
\onecolumn
\input{Appendix_Regret}

\end{document}

%% file: Introduction.tex
\section{Introduction}\label{Section_Introduction}
The celebrated Kalman filter has been a fundamental approach for estimation and prediction of time-series data, with diverse applications ranging from control systems and robotics~\citep{bertsekas2017dynamic, durrant2006simultaneous} to computer vision~\citep{coskun2017long} and economics~\citep{harvey1990forecasting}. 
Given a known system model with known noise statistics, the Kalman filter predicts future observations of a \emph{partially observable} dynamical process by filtering past observations. When the underlying process is linear and the noise is Gaussian, the Kalman filter is optimal in the sense that it minimizes the mean square prediction error. Since Kalman's seminal paper~\citep{kalman1960new}, the stability and statistical properties of the Kalman filter have been well studied when the system model is known.  

Learning to predict unknown partially observed  systems is a significantly more challenging problem. Even in the case of linear systems, learning directly the model parameters of the system results in nonlinear, non-convex problems~\citep{yu2018identification}. 
Adaptive filtering algorithms address the problem of making observation predictions when the system model or the noise statistics are unknown or changing~\citep{ljung1978convergence,moore1979multivariable,lai1991recursive,ding2006performance}. 
These adaptive filtering approaches are usually based on variations of extended least squares. 
Despite the importance of adaptive filtering in applications such as GPS, the regret of online filtering algorithms has not been considered in this classical literature.  

In this paper, we consider the problem of predicting observations generated by an {\em unknown}, partially observable linear dynamical system in state-space form. We assume that the system dynamics and observation map are corrupted by Gaussian noise. 
Our goal is to find an online prediction algorithm that has provable regret bounds with respect to the Kalman filter that has access to the full system model.
Our technical contributions are:

{\bf System theoretic regret:} We define a notion of regret that has a natural, system theoretic interpretation.  The prediction error of an online prediction algorithm is compared against the prediction error of the Kalman filter that has access to the exact model, which is allowed to be arbitrary. Previous regret definitions~\citep{kozdoba2019line} required the model to lie in a finite set. 

{\bf Logarithmic regret for the Kalman filter:} We present the first online prediction algorithm with provable logarithmic regret upper bounds for the classical Kalman filter.  In fact, we prove that with high probability the regret of our algoritm is of the order of $\tilde{O}(1)$, where $\tilde{O}$ hides $\mathrm{poly}\log N$ terms, where $N$ is the number of observations collected. 
Our algorithm has polynomial time complexity, requires linear memory, and is based on subspace system identification techniques~\citep{qin2006overview}. Instead of optimizing over the state-space parameters, which is a non-convex problem, we  convexify the problem by establishing an approximate regression between the next observation and past observations.  Our analysis is based on the stability properties of the Kalman filter, tools for self-normalized martingales and matrices, and additional results for persistency of excitation developed in this paper.

{\bf Logarithmic regret for non-explosive systems:} Our regret guarantees hold for the class of non-explosive systems, which includes marginally stable linear systems as well as as systems where the  state can grow at a polynomial rate. This settles an open question and conludes that online prediction performance does not depend on the system stability gap\footnote{$1/(1-\rho)$, where $\rho$ is the spectral radius of the system~\citep{simchowitz2018learning}} of the system. Although it was recently shown that the stability gap does not affect system identification~\citep{simchowitz2019semi}, whether the stability gap affects online prediction under stochastic noise was an open problem.

{\bf Regret analysis for other predictors:} 
Our approach directly carries over to various interesting online predictors.  For example, our analysis can be directly extended to the case of $f-$step ahead prediction of observations.  Another extension focuses on the regret of hidden state predictors when the state-space basis representation is known a priori. The latter situation arises, for example, when the state-space model is known but the noise statistics are unknown.  All these predictors enjoy similar logarithmic regret bounds.

{\bf Gap between model-free LQR and Kalman filter:} One of the implications of our bounds is that learning to predict observations like the Kalman filter is provably easier than solving the online Linear Quadratic Regulator (LQR) problem, which in general requires $O(\sqrt{N})$ regret. In fact, recent results suggest that in the LQR case, the regret is lower bounded by $\Omega(\sqrt{N})$~\citep{simchowitz2020naive}. This might not be surprising due to the fact that, in the absence of exogenous inputs, we cannot inject exploratory signals into the system.

\subsection{Related work}
Recently, there have been very important results  addressing the regret of the adaptive Linear Quadratic Regulator (LQR) problem~\citep{abbasi2011regret,faradonbeh2017optimism,ouyang2017learning,abeille2018improved,dean2018regret,mania2019certainty,cohen2019learning}. The best regret for LQR is sublinear and of the order of $\tilde{O}(\sqrt{N})$, where $N$ is the numbers of state samples collected; an in-depth survey can be found in~\cite{matni2019self}. When the system model is {\em known}, then the Kalman filter is the dual of the Linear Quadratic Regulator, suggesting that this duality can be exploited in deriving the regret of the Kalman filter. However, when the system model is {\em unknown}, the Linear Quadratic Regular and the Kalman filter are not dual problems~\citep{tsiamis2019sample}.  As the state is fully observed in LQR, the system identification in adaptive LQR reduces to a simple least squares problem.  In the adaptive Kalman filter, the state is {\em partially} observed resulting in non-convex system identification problems 
 requiring us to consider a different approach.

A related but different problem focuses on online prediction algorithms for systems without internal states (such as ARMA - autoregressive moving average)~\citep{anava2013online}. 
Prediction of observations generated by state space models in the case of exogenous inputs and adversarial noise but with a bounded budget was studied in~\cite{hazan2018spectral}.  
The work closest to ours is the very recent work of~\cite{kozdoba2019line}, 
where 
 regret bounds with respect to the Kalman Filter was studied for the first time but in the restricted context of scalar and bounded observations. The regret is shown to be linear, where the linear term is small but nonzero. 
 
Our online algorithm is inspired by subspace identification techniques~\citep{bauer1999consistency}. The technical approach is based on classical results for the analysis of the least-squares estimator for time series~\citep{lai1982least}, high-dimensional statistics~\citep{vershynin2018high} as well as modern results for finite sample analysis of system identification in both the fully observed~\citep{faradonbeh2018finite,simchowitz2018learning,sarkar2018fast} and the partially observed case~\citep{hardt2018gradient,oymak2018non,simchowitz2019semi,sarkar2019finite,tsiamis2019finite}.

\textbf{Paper organization.} In Section~\ref{Section_Formulation} we provide some background on the classical  Kalman Filter and formulate the regret problem considered in this paper. In Section~\ref{Section_Algorithm} we introduce the online learning algorithm while the regret guarantees are presented in Section~\ref{Section_Analysis}. We conclude with generalizations  and discussion of future work in Sections~\ref{Section_Extensions},~\ref{Section_Discussion}. Detailed proofs can be found in the Appendix.\\
\textbf{Notation. } With $\norm{}_2$ we denote the Euclidean norm for vectors and the spectral norm for matrices. The spectral radius of a matrix $A$ is denoted by $\rho\paren{A}$.
The smallest singular value of a matrix $A$ is denoted by $\sigma_{\min}(A)$. By $A^*$ we denote the transpose of $A$. Unless explicitly stated, when using the standard $O(N),\tilde{O}(N),o(N)$ notation we hide all other quantities, e.g. system constants, system dimensions, logarithms of failure probabilities. The $\tilde{O}(N)$ notation hides (powers of) logarithmic terms of $N$. The $\mathrm{poly}(x)$ notation means a polynomial function of $x$.

%% file: Formulation.tex
\section{Problem Formulation}\label{Section_Formulation}
The Kalman filter considers the problem of predicting observations generated by the following \emph{state-space system}:
\begin{equation}\label{FOR_EQN_System_Original}
\begin{aligned}
x_{k+1}&=Ax_{k}+w_{k},&&\:w_k\stackrel{\text{i.i.d.}}{\sim}\mathcal{N}\paren{0,Q}\\
y_{k}&=Cx_k+v_k,&&\:v_k\stackrel{\text{i.i.d.}}{\sim}\mathcal{N}\paren{0,R}
\end{aligned}
\end{equation}
where $x_k\in\R^n$ is the state, $y_k,\in \R^m$ are the observations (outputs), $A\in\R^{n\times n}$ is the system matrix and $C\in\R^{m\times n}$ is the observation matrix. The time series $w_k,v_k$ represent the process and measurement noise respectively and are modeled as zero mean i.i.d. Gaussian variables, independent of each other, with covariances $Q$ and $R$ respectively. The initial state is zero mean Gaussian with covariance $\Sigma_0$ and independent of the noises. The following assumption holds throughout this paper.

\begin{assumption}
	System~\eqref{FOR_EQN_System_Original} is non-explosive\footnote{This class includes marginally stable systems as well as systems with polynomial state growth.}, namely the spectral radius is $\rho(A)\le 1$.
\end{assumption}

Let $\F_k\triangleq \sigma(y_{0},\dots,y_k)$ be the filtration generated by the observations
$y_{0},\dots,y_k$. Given the observations up to time $k$, the optimal prediction $\hat{y}_{k+1}$ at time $k+1$ in the minimum mean square error (mmse) sense is defined as:
\begin{equation}\label{FOR_EQN_MMSE}
\hat{y}_{k+1}\triangleq\arg\min_{z\in\F_k}\mathbb{E}\clint{\norm{y_{k+1}-z}^2_2|\F_k}.
\end{equation}
In the case of system~\eqref{FOR_EQN_System_Original},
the optimal predictor 
admits a recursive expression, known as the \emph{Kalman filter}:
\begin{equation}\label{FOR_EQN_System_Innovation}
\begin{aligned}
\hat{x}_{k+1}&=A\hat{x}_{k}+Ke_k,\,\hat{x}_0=0\\
\hat{y}_{k+1}&=C\hat{x}_{k+1}\\
y_k&=C\hat{x}_k+e_k
\end{aligned}
\end{equation}
where $e_k\triangleq y_k-C\hat{x}_k $ is the innovation noise process. Matrix $K\in\R^{n\times m}$ is called the Kalman filter gain, and can be computed based on $A,C,Q,R$--see ~\eqref{FOR_EQN_Riccati} in Subection~\ref{Section_Formulation_Technical}.

Although the Kalman filter gives the optimal mmse prediction, it requires the system matrices $A,C$ and noise covariances $Q,R$ to be known. In this paper, we seek online learning algorithms that can predict observations based only on past observation data, without any knowledge of system matrices of noise covariances. 
To quantify the online prediction performance, we define the regret of our online learning algorithm with respect to the Kalman filter~\eqref{FOR_EQN_System_Innovation} that has full knowledge of system model~(\ref{FOR_EQN_System_Original}). Our goal is to achieve sublinear regret, as defined in the following problem statement.
\begin{problem}
	Assume that $A,C,Q,R$  in system model~(\ref{FOR_EQN_System_Original}) are unknown. Consider a sequence $y_0,y_1\dots$ of observations generated by system~\eqref{FOR_EQN_System_Original}. Let $\tilde{y}_k\in \F_{k-1}$ be the prediction of an online learning algorithm based on the history $y_{k-1},\dots,y_0$ and $\hat{y}_k$ be the Kalman filter prediction~\eqref{FOR_EQN_System_Innovation} that has full knowledge of model~(\ref{FOR_EQN_System_Original}). Define the regret:
	\begin{equation}\label{FOR_EQN_Regret_Output}
	\Reg_{N}\triangleq \sum^{N}_{k=1}\snorm{y_k-\tilde{y}_k}^2-\sum^{N}_{k=1}\snorm{y_k-\hat{y}_k}^2
	\end{equation} 
Fix a failure probability $\delta>0$. Our goal is to find a learning algorithm such that with probability at least $1-\delta$:
	\[
	\Reg_{N}\le \mathrm{poly}(\log1/\delta)o(N),
	\]
	where $o(N)$ does not depend on $\delta$.
\end{problem}
Our regret definition has a natural system theoretic interpretation since it is defined with respect to the Kalman filter. In Section~\ref{Section_Extensions}, we discuss an alternative regret definition.

In the following subsection we provide some background on the Kalman filter and specify some standard assumptions, which guarantee that the Kalman filter is well-defined.

\subsection{Kalman Filter Background}\label{Section_Formulation_Technical}
The Kalman filter enjoys two critical properties, namely closed-loop stability and innovation orthogonality, that are now  reviewed.
The following standard assumption holds throughout the paper and guarantees that the Kalman filter is well-defined.
\begin{assumption}\label{ASS_Kalman}
	The system matrix pair $(A,C)$ is observable, i.e. the observability matrix:
	\begin{equation}\label{FOR_EQN_Observability_Matrix}
	\O_k\triangleq \matr{{cccc}C^*&A^*C^*&\dots&(A^*)^{k-1}C^*}^*
	\end{equation}
	has rank $n$ for all $k\ge n$. The pair $(A,Q^{1/2})$ is controllable, i.e. the controllability matrix
		\begin{equation}\label{FOR_EQN_Controllability_Matrix}
		\matr{{cccc}Q^{1/2}&AQ^{1/2}&\dots&A^{k-1}Q^{1/2}}
		\end{equation}
		has rank $n$ for all $k\ge n$, 
	and $R$ is strictly positive definite.
\end{assumption}

The following result shows that under Assumption~\ref{ASS_Kalman}, the closed loop matrix $A-KC$ of the Kalman filter is stable. 
\begin{proposition}[\citealt{anderson2005optimal}]\label{FOR_PROP_KFStab}
	Consider system~\eqref{FOR_EQN_System_Original} under Assumption~\ref{ASS_Kalman}. The Kalman filter gain in~(\ref{FOR_EQN_System_Innovation}) is computed by:	
	\begin{equation*}
	\begin{aligned}
	K&=APC^*\paren{CPC^*+R}^{-1},
	\end{aligned}
	\end{equation*}  
	where $P$ is the positive definite solution to
	\begin{equation}\label{FOR_EQN_Riccati}
	P=(A-KC)P(A-KC)^*+Q+KRK^*.
	\end{equation}
	Moreover, the closed-loop matrix $A-KC$ is stable, i.e. it has spectral radius $\rho(A-KC)<1$. 
\end{proposition}
Proposition~\ref{FOR_PROP_KFStab} implies that the Kalman filter reaches steady state exponentially fast, allowing us to assume the following.
\begin{assumption}\label{ASS_Stationarity_Kalman}
	We assume that  the initial state covariance is $\Sigma_0=P$, where $P$ is defined in~\eqref{FOR_EQN_Riccati}. 
\end{assumption} 
If $\Sigma_0\neq P$, then we have to consider time-varying gains $K_k$ in~\eqref{FOR_EQN_System_Innovation}.  The condition $\Sigma_0=P$ guarantees that the Kalman filter~\eqref{FOR_EQN_System_Innovation} has stabilized to its steady-state so that the gain $K$ is constant. Since the Kalman filter converges exponentially fast to its steady-state~\cite{anderson2005optimal}, this is a very mild assumption; it is also standard~\cite{knudsen2001consistency}.  

The next assumption makes sure that system~\eqref{FOR_EQN_System_Innovation} is minimal.
\begin{assumption}
	The pair $(A,K)$ is controllable.
\end{assumption}
If the pair $(A,K)$ is not controllable, then we can find a similarity transformation $\bar{x}_k=S\hat{x}_k$ such that:
\begin{align*}
\matr{{c}\bar{x}_{1,k+1}\\\bar{x}_{2,k+1}}=\matr{{cc}\bar{A}_{11}&\bar{A}_{12}\\0&A_{22}}\bar{x}_k+\matr{{c}K_1\\0}\bar{e}_k.
\end{align*}
But since $\hat{x}_k=0$ this implies that $\bar{x}_{2,k}=0$ for all $k\ge 0$. Hence we could remove $\bar{x}_{2,k}$ and consider a reduced system representation with only $\bar{x}_{1,k+1}$.

The following assumption is for notational simplicity.  It assumes that the largest eigenvalue of  $A-KC$ is simple.
\begin{assumption}\label{ASS_Bounded_Responses}
	For some $M>0$ and all $t\ge 0$, the closed-loop matrix satisfies $\norm{(A-KC)^t}_2\le M\rho(A-KC)^t$. 
\end{assumption}
If the largest eigenvalue has larger multiplicity then we can just consider $\rho(A-KC)+\epsilon$ in the above bound, for sufficiently small $\epsilon$.

In addition to the previous stability properties , the other nice property of the Kalman Filter is that the innovation sequence $e_k=y_k-\hat{y}_k$ is orthogonal (uncorrelated) and, by Gaussianity, also i.i.d. 
By the law of large numbers, this implies that the $\ell_2$ accumulative error $\sum_{k=0}^{N}\norm{y_k-\hat{y}_k}^2_2$ will be of the order of $O(N)$ almost surely. Predicting the true observations exactly is impossible in the stochastic noise setting, even if we know the system model.

Note that both systems~\eqref{FOR_EQN_System_Original},~\eqref{FOR_EQN_System_Innovation} can generate the same observations~$y_k$, i.e. the noise parameterization is not unique~\cite{van2012subspace}.
Another source of ill-posedness is that the state space parameterization is non-unique. Any similarity transformation $S^{-1}AS$, $CS$, $S^{-1}QS^{-*}$ generates the same observations. In the following section, we will address these problems by considering an alternative system representation.

%% file: Output_Prediction.tex
\section{Online Prediction Algorithm}\label{Section_Algorithm}
The main idea of our online prediction algorithm is based on a system representation that has been used in the subspace system identification~\cite{bauer1999consistency}. 
 Let $p$ be an integer that represents how far we look into the past. We define the vector of past observations at time $k$:
\begin{align}\label{OUT_EQN_future_past_output}
Z_{k,p}\triangleq\matr{{ccc}y^*_{k-p}&\dots&y^*_{k-1}}^*,\,k\ge p.
\end{align}
Define also the matrix of closed-loop responses 
\begin{equation}
\label{EQN_Kalman_Controllability}
G_p\triangleq\matr{{ccc}C(A-KC)^{p-1}K&\cdots&CK}
\end{equation}
By expanding the Kalman filter ~\eqref{FOR_EQN_System_Innovation} $p$-steps into the past, the observation at time $k$ can be rewritten as
\begin{equation}\label{OUT_EQN_Linear_Regression}
y_k=G_p Z_{k,p}+\underbrace{C(A-KC)^{p}\hat{x}_{k-p}}_{\text{bias}}+e_k.
\end{equation}
Instead of optimizing over system parameters $A,C,K$, which results in a non-convex optimization problem, we optimize over (the higher dimensional) $G_p$, which makes the problem convex. From an online learning perspective, this technique is also known as improper learning.
Using this lifting, we can learn a least squares estimate $\tilde{G}_{k,p}$ by regressing outputs $y_t$ to past outputs $Z_{t,p}$ for $t\le k$:
\begin{equation}\label{OUT_EQN_G_estimate}
\tilde{G}_{k,p}=\sum_{t=p}^{k}y_tZ^*_{t,p}\paren{\lambda I+\sum_{t=p}^{k}Z_{t,p}Z^*_{t,p}}^{-1},
\end{equation} 
where $\lambda>0$ is a regularization parameter.
Then, to predict the next observation, we can compute:
\begin{equation}\label{OUT_EQN_Y_prediction}
\tilde{y}_{k+1}=\tilde{G}_{k,p}Z_{k+1,p}.
\end{equation} 
We could also use the recursive update:
\begin{align*}
\bar{V}_{k,p}&=\bar{V}_{k-1,p}+Z_{k,p}Z_{k,p}^*\\
\tilde{G}_{k,p}&=\tilde{G}_{k-1,p}+(y_k-\tilde{y}_k)Z^*_{k,p}\bar{V}^{-1}_k
\end{align*}
as long as the past $p$ is kept constant.
\begin{algorithm}[!t]
	\caption{Online Prediction Algorithm}
	\label{OUT_ALG_Output}
	\begin{algorithmic}
		\STATE {\bfseries Input:} $\beta$, $\lambda$, $T_{\text{init}}$ such that $T_{\text{init}}>\beta \log T_{\text{init}}$
			\FOR {$\text{k}=0,\dots,T_{\text{init}}-1$}
		\STATE Observe $y_{k}$, $\tilde{y}_k=0$
		\ENDFOR
		\FOR {$\text{i}=1,2,\dots$}
		\STATE $T=2^{i-1}T_{\text{init}}$
		\STATE $p=\beta\log T$
		\STATE $\bar{V}_{T-1}=\lambda I+\sum_{t=p}^{T-1}Z_{t,p}Z_{t,p}^*$
		\STATE $\tilde{G}_{T-1}=\paren{\sum_{t=p}^{T-1}y_tZ^*_{t,p}}\bar{V}_{T-1}^{-1}$
		\FOR {$k=T,\dots,2T-1$}
		\STATE Predict $\tilde{y}_{k}=\tilde{G}_{k-1}Z_{k,p}$
		\STATE Observe $y_{k}$
		\STATE Update $\bar{V}_{k}=\bar{V}_{k-1}+Z_{k,p}Z^*_{k,p}$
		\STATE $\tilde{G}_{k}=\tilde{G}_{k-1}+\paren{y_{k}-\tilde{y}_{k}}Z^*_{k,p}\bar{V}_{k}^{-1}$
		\ENDFOR
		\ENDFOR
	\end{algorithmic}
\end{algorithm}

Due to the stability properties of the Kalman filter (Section~\ref{Section_Formulation_Technical}), if we consider $p$ past observations, then the bias term in equation~\eqref{OUT_EQN_Linear_Regression} is of the order of $\rho(A-KC)^{p}\norm{\hat{x}_{x-p}}_2$. Notice that for non-explosive systems the state $\hat{x}_{k-p}$ can grow polynomially fast in the worst case. Even if $\hat{x}_{k-p}$ remains bounded, keeping the past $p$ constant would lead to a non-vanishing bias error (linear regret). Thus, to make sure that the prediction error decreases, we need to gradually increase the past horizon $p$. 
For this reason, inspired by the ``doubling trick"~\cite{cesa2006prediction}, we divide the learning in epochs, where every epoch is twice longer than the previous one. During every epoch we keep the past horizon constant. 
Since $\rho(A-KC)^{p}$ is exponentially decreasing, it is sufficient to slowly increase the past as $p=O(\log T)$, where $T$ is the epoch duration.

The pseudo-code of our online prediction approach can be found in Algorithm~\ref{OUT_ALG_Output}. 
 Each epoch lasts from time $T_i,\dots,2T_{i}-1$, where $i=1,\dots,$ is the epoch, $T_i=2^{i-1}T_{\text{init}}$, and $T_{\text{init}}$ is a design parameter (the length of the first epoch). During every epoch, we keep the past $p_i=\beta \log (T_i)$ constant, where $\beta$ is a design parameter. Initially, from time $0$ to $T_{\text{init}}-1$, we have a warm-up phase where we gather enough observations to start predicting. To make sure  that $p_i<T_i$, we tune $T_{\text{init}}$ accordingly. 
Within an epoch, the least squares based predictor~\eqref{OUT_EQN_Y_prediction} can be implemented in a recursive way, which requires polynomial complexity and at most $O(\log T_i)$ memory. In the beginning of an epoch, when $p_i$ is updated, we re-initialize the recursive predictor based on the whole past, which requires polynomial complexity and $O(T_i)$ memory. 
Hence, in total, after $N$ collected samples, the computational complexity is polynomial and the memory requirement is $O(N)$.
In Section~\ref{Section_Discussion}, we discuss ways to modify the initialization when changing epochs without using the whole past, which can reduce the memory to $O(\log N)$. 

An important property of Algorithm~\ref{OUT_ALG_Output} is that no knowledge about the dynamics or even the state dimension~$n$ is required. Note that there is a tradeoff between the bias error and statistical efficiency. Increasing the past horizon by selecting larger $\beta$ leads to smaller bias error, but increases the sample complexity of learning $G_{p}$ since we have more unknowns; it is also harder to achieve persistency of excitation, i.e. to have a large enough smallest singular value of $\bar{V}_k$. 

%% file: Regret_Analysis.tex
\section{Regret Analysis}\label{Section_Analysis}
In this section, we prove that with high probability the prediction regret is not only sublinear, but also of the order of $\mathrm{poly}\log N$ (or $\tilde{O}(1)$), where $N$ is the number of observations collected so far.
The challenge in the non-explosive regime is that the observations grow unbounded polynomially fast ($\Omega(\sqrt{N})$). Meanwhile, recent work in finite sample analysis of system identification~\cite{oymak2018non,simchowitz2019semi,tsiamis2019finite,sarkar2019finite} shows that the model parameters can be learned at a slower rate ($O(1/\sqrt{N})$). Therefore these system identification results cannot be directly applied to obtain regret bounds for our problem.
Nonetheless, we show that our online Algorithm~\ref{OUT_ALG_Output} mitigates the effect of unbounded observations. As a result, the logarithmic regret bound of $\tilde{O}(1)$ remains valid even as we approach instability. 

We provide two results, one for non-explosive systems ($\rho(A)\leq1$) and one for stable systems $(\rho(A)<1)$. 
Before we present the regret results, let us introduce some standard notions. Let $a(s)=s^d-a_{d-1}s^{d-1}\dots-a_0$ be the minimal polynomial of matrix $A$, i.e. the minimum degree polynomial such that $a(A)=0$. Denote its degree by $d$. We define the $\ell_1$ norm of its coefficients as $\norm{a}_{1}\triangleq \sum_{i=0}^{d-1}\abs{a_i}$; the $\ell_2$ norm $\norm{a}_2$ is defined in a similar way.
 Let  $\kappa$ be the dimension of the largest Jordan block of $A$ that is a associated with an eigenvalue on the unit circle (i.e. $\rho(A) = 1$). Let $\kappa{\max}$ be the largest Jordan block among all eigenvalues. In general, $\kappa\le \kappa_{\max}\le d\le n$.

\begin{theorem}[Regret for non-explosive systems]\label{OUT_THM_Main_Bound}
    Consider system~\eqref{FOR_EQN_System_Innovation} with $\rho(A)\leq 1$. 
	Let $y_0,\dots,y_N$ be sequence of system observations  with $\hat{y}_0,\dots,\hat{y}_N$ being the respective Kalman filter predictions. Let $\tilde{y}_0,\dots,\tilde{y}_N$ be the predictions of Algorithm~\ref{OUT_ALG_Output} with \begin{equation}\label{OUT_EQN_beta}
	\beta=\Omega\paren{\frac{\kappa}{\log(1/\rho(A-KC))}}
	\end{equation}
	and fix a failure probability $\delta>0$. There exists a $N_0=\mathrm{poly}\paren{n,\beta,\kappa,\log1/\delta}$, independent of $N$, such that with probability at least $1-\delta$, if $N> N_0$ then:
	\begin{equation}\label{OUT_EQN_Main_Bound}
	\begin{aligned}
\Reg_N\le& \mathrm{poly}(d^{\kappa_{\max}},n,\beta,\norm{a}_2,\kappa,\log \frac{1}{\delta})\tilde{O}(1)
	\end{aligned}
	\end{equation}
where $\tilde{O}(1)$  hides $\tilde{O}(N_0)$ and $\tilde{O}(T^{2\kappa}_{\text{init}})$ terms.
\end{theorem}

Theorem~\ref{OUT_THM_Main_Bound} provides the first logarithmic regret upper bounds for the general problem of Kalman filter prediction. The burn-in time $N_0$ is related to persistency of excitation conditions, i.e. initially we need enough samples to guarantee that the smallest singular value of the Gram matrix $\bar{V}_k$ increases linearly with $k$. Our bounds do not depend on the stability gap $1/(1-\rho(A))$ and they do not degrade as we approach instability. However, they suggest, via $\beta$, that the stability radius $\rho(A-KC)$ of the Kalman filter closed-loop matrix affects the difficulty of learning.

Interestingly, our bounds show that the problem of learning to predict observations like the Kalman filter is provably easier than the online LQR, in the case of unknown model. The latter requires in general regret of the order of $\sqrt{N}$~\citep{simchowitz2020naive}. This is another reason why the problems are not dual in the unknown model case. This gap might be expected since in the case of Kalman filter without exogenous inputs, there is no exploratory signal.

The upper bound also depends on the quantities $d^{\kappa_{\max}}$ and $\norm{a}_1$, both of which can be exponential in the dimension of the system state $n$ in the worst case. This can happen, for example, if $\kappa=n$, i.e. the system is an $n-$th order integrator. 
Dependence of learning performance on the coefficients of the characteristic or minimal polynomial has been found in related work~\citep{hardt2018gradient}. This dependence can be improved in some cases--see for example the phase polynomial in~\cite{hazan2018spectral}, where there are no repeated eigenvalues. In our case, this dependence could perhaps be improved by applying the techniques of~\cite{simchowitz2019semi}. 
However, it is an open question whether it is possible to avoid the exponential dependence on $\kappa$, $\kappa_{\max}$. It might be possible that systems with long chain structure, e.g. integrators,  are indeed harder to learn. In system theory it is known that even in the known model case, such systems might be difficult to observe. In open-loop system identification~\citep{simchowitz2019semi}, such a dependence also appears. 
It might be an inherent limitation of the problem, since fundamental quantities of the system, for example matrix $A^i$ or the observability matrix $\O_i$ scale with $i^{\kappa}$.

Both of the above issues are avoided in the case of stable systems ($\rho(A)<1$),
where we have the following result.
\begin{theorem}[Regret for stable systems]\label{STABLE_THM_Main_Bound}
	Consider system~\eqref{FOR_EQN_System_Innovation} with $\rho(A)<1$. 
Let $y_0,\dots,y_N$ be sequence of system observations  with $\hat{y}_0,\dots,\hat{y}_N$ being the respective Kalman filter predictions. Let $\tilde{y}_0,\dots,\tilde{y}_N$ be the predictions of Algorithm~\ref{OUT_ALG_Output} with $
\beta=\Omega\paren{\frac{1}{\log(1/\rho(A-KC))}}.
$
	Fix a failure probability $\delta>0$. Then there exists a 
	\[
	N_0=\mathrm{poly}\paren{n,\beta,\log 1/\rho(A),\log1/\delta}
	\] 
	such that with probability at least $1-\delta$, if $N> N_0$ then:
	{ \medmuskip=0mu\thickmuskip=1mu
		\begin{equation}\label{STABLE_EQN_Main_Bound}
		\Reg_N\le \mathrm{poly}(n,\beta,\log \frac{1}{\delta})\tilde{O}(1)
		\end{equation}}
where $\tilde{O}(1)$ hides $\tilde{O}(N_0)$ and $\tilde{O}(T_{\text{init}})$ terms.
\end{theorem}
Notice that for stable systems we no longer have quantities that depend exponentially on the dimension $n$. The main bound~\eqref{STABLE_EQN_Main_Bound} does not depend on the stability gap $1/(1-\rho(A))$ However, via $N_0$, the guarantees depend logarithmically on the inverse radius $\log 1/\rho(A)$. This quantity is related to the time needed for a stable system to approach stationarity.

The proofs of Theorem~\ref{OUT_THM_Main_Bound} and Theorem~\ref{STABLE_THM_Main_Bound} can be found in the Appendix.  In the next subsection, we provide an overview of the regret analysis and explain why the quantities $d^{\kappa_{\max}}$ and $\norm{a}_{2}$ appear in the bound in Theorem~\ref{OUT_THM_Main_Bound}. We also explain what changes in the case of stable systems addressed by Theorem~\ref{STABLE_THM_Main_Bound}.

\subsection{Regret  analysis overview}
Recall the definition of the innovation error $e_k=y_k-\hat{y}_k$. For brevity, we also define the error  $\tilde{e}_k\triangleq\tilde{y}_k-\hat{y}_k$ between the online prediction of Algorithm~\ref{OUT_ALG_Output} and the Kalman Filter prediction. By adding and subtracting $\hat{y}_k$ in the first term, we obtain
\begin{align*}
\Reg_N&=\sum_{k=1}^{N}\norm{e_k+\hat{y}_k-\tilde{y}_k}^2_2-\norm{e_k}^{2}_2\\
&=\underbrace{\sum_{k=1}^{N}\norm{\hat{y}_k-\tilde{y}_k}^2_2}_{\LL_N}+2\underbrace{\sum_{k=1}^{N}e^*_k\paren{\hat{y}_k-\tilde{y}_k}}_{\text{martingale term}}
\end{align*}
It is sufficient to prove that the square loss $\ell_2$:
\begin{equation}
\LL_N\triangleq\sum_{k=1}^{N}\norm{\hat{y}_k-\tilde{y}_k}^2_2
\end{equation} 
is logarithmic in $N$. Because the innovations are i.i.d., we have a martingale structure for the second term since $e_k\in\F_k$, while $\tilde{e}_k\in\F_{k-1}$.
 The martingale term will in general be small
and can be bounded in terms of the square loss $\LL_N$. 
In particular, the quantity \[(\LL_N+1)^{-1/2}\sum_{k=1}^{N}e^*_k\paren{\hat{y}_k-\tilde{y}_k}\] is a self-normalized martingale and can be analyzed based on the techniques of~\cite{abbasi2011improved,sarkar2018fast}, which imply that  
\[
\sum_{k=1}^{N}e^*_k\paren{\hat{y}_k-\tilde{y}_k}=\tilde{O}(\sqrt{\LL_N})
\]
with high probability.
Hence, we will focus on bounding the square loss $\LL_N$. 

For the remaining section,
we will assume that we are within one epoch $i$ so that the past horizon $p=p_i$ and $T=2^{i-1}T_{init}$ are kept constant. For brevity, we omit the subscript $p$ from all variables and write $G,\tilde{G}_k,Z_k$ instead of $G_p,\tilde{G}_{k,p},Z_{k,p}$.

Define $S_{k-1}\triangleq \sum_{i=p}^{k-1}e_{i}Z^*_{i}$ and $\bar{V}_{k-1}\triangleq \lambda I+\sum_{i=p}^{k-1}Z_{i}Z^*_{i}$.
Then, the error between our online prediction and the Kalman filter prediction can be written as:
\begin{align}
\tilde{e}_k&=(\tilde{G}_{k-1}-G)Z_k-C(A-KC)^p\hat{x}_{k-p}
\nonumber\\&=\underbrace{S_{k-1}\bar{V}^{-1}_{k-1}Z_{k}}_{\text{regression}}+\underbrace{\lambda G\bar{V}^{-1}_{k-1}Z_{k}}_{\text{regularization}}\label{OUT_EQN_error_expression}\\
&+\underbrace{C(A-KC)^p\paren{\sum_{i=T}^{k-1}\hat{x}_{i-p}Z^*_i\bar{V}^{-1}_{k-1}Z_k-\hat{x}_{k-p}}}_{\text{truncation bias}}\nonumber.
\end{align}
The regression term is due to the noise $e_k$ entering the system. The truncation bias is due to using only $p$ past observations and not all of them. 

The key ingredients to analyze the cumulative error $\LL_N$ are i) the stability properties of the closed-loop matrix $A-KC$; ii) self-normalization properties of predictor~\eqref{OUT_EQN_Y_prediction};   and iii) persistency of excitation for the past observations with high probability. By persistency of excitation we mean that the least singular value of the Gram matrix $\bar{V}_k$ is increasing as fast as $O(k)$ with high probability.

\textbf{Regression term.} We can rewrite the regression term as a product of two separate terms:
\[
S_{k-1}\bar{V}^{-1}_{k-1}Z_{k}=(S_{k-1}\bar{V}^{-1/2}_{k-1})(\bar{V}^{-1/2}_{k-1}Z_{k}).
\]
The first term, $S_{k-1}\bar{V}^{-1/2}_{k-1}$ is again a self-normalized martingale and can be analyzed based on the techniques of~\cite{abbasi2011improved,sarkar2018fast}, which imply that  the term $\sup_{T\le k\le 2T-1}\norm{S_{k-1}\bar{V}^{-1/2}_{k-1}}^2_2$ grows logarithmically with $T$. The martingale property again comes from the fact that the innovation process $e_k$ of the Kalman Filter is  i.i.d.--see Section~\ref{Section_Formulation_Technical}.

The second term, $\bar{V}^{-1/2}_{k-1}Z_{k}$, is almost self-normalized since $\bar{V}_{k-1}$ is the Gram matrix of $Z_{k-1},\dots,Z_p$. It could be bounded using the following lemma which is inspired by~\cite{lai1982least}.
\begin{lemma}\label{OUT_LEM_self_normalization}
	Let $\bar{V}_{k-1}=\lambda I+\sum_{i=p}^{k-1}Z_iZ^*_i$. Then, the following inequality holds:
	\[
	\sum_{k=T+1}^{2T}Z_{k-1}\bar{V}^{-1}_{k-1}Z_{k-1}\le \log\det (\bar{V}_{2T-1}\bar{V}^{-1}_{T-1})
	\]
\end{lemma}
The intuition is that $Z_{k-1}Z_{k-1}^*$ appears in $\bar{V}_{k-1}$ and, hence, it cancels out the effect of  $Z_{k-1}$. 
Unfortunately, we cannot directly use the above inequality for $\bar{V}^{-1/2}_{k-1}Z_{k}$ since $Z_kZ^*_k$ is not explicitly contained in $\bar{V}_{k-1}$.
However, we can exploit the fact that $Z_k$ does not change too fast compared to the most recent past $Z_{k-1}$, $\dots$, $Z_{k-n}$. 
\begin{lemma}[ARMA-like representation]\label{OUT_LEM_ARMA_representation}
	Let $y_{0},y_1\dots$ be observations generated by system~\eqref{FOR_EQN_System_Original}. Fix a past horizon $p$ and let $a$ be the minimal polynomial of $A$ with degree $d$. Then, the past observations satisfy the following recursion
\begin{equation}\label{OUT_EQN_ARMA}
Z_{k}=a_{d-1}Z_{k-1}+\dots+a_{0}Z_{k-d}+\delta_k,
\end{equation}
where $\delta_k\in\F_{k-1}$ with
\begin{equation}
\norm{\delta_k}_2\le \Delta \sup_{i\le k-1}\norm{e_i}_2,
\end{equation}
where $\Delta=O(d^{\kappa_{\max}-1}\norm{a}_1\sqrt{p})$
\end{lemma}
Intuitively, the unbounded components of $Z_{k}$ are captured by the recent history $Z_{k-1}$, $\dots$, $Z_{k-d}$ and the residual $\delta_k$ is bounded.
Replacing $Z_k$ with~\eqref{OUT_EQN_ARMA} we obtain by two Cauchy-Schwarz inequality applications: \begin{align*}
\norm{\bar{V}^{-1/2}_{k-1}Z_{k}}^2_2\le 2\norm{a}^2_2\sum_{i=0}^{d-1}\norm{\bar{V}^{-1/2}_{k-1}Z_{k-d+i}}^2_{2}+2\norm{\bar{V}^{-1/2}_{k-1}\delta_k}^2_2.
\end{align*}
The terms $\bar{V}^{-1/2}_{k-1}Z_{k-d+i}$ in the sum are now indeed normalized and can be bounded using Lemma~\ref{OUT_LEM_self_normalization}. For $\bar{V}^{-1/2}_{k-1}\delta_k$ we exploit a new persistency of excitation result.
\begin{lemma}[Uniform Persistency of Excitation]\label{Lemma_PE}
	Consider the conditions of Theorem~\ref{OUT_THM_Main_Bound}. Select a failure probability $\delta>0$. Let $T=2^{i-1}T_{\text{init}}$ for some fixed epoch $i$ with $p=\beta \log T$ the corresponding past horizon.
	There exists a $N_0=\mathrm{poly}(n,\beta,\kappa,\log 1/\delta)$ such that if $T\ge N_0$, then with probability at least $1-\delta$:
	\begin{align}
\sum_{j=p}^{k}Z_{j}Z^*_{j}\succeq \frac{k-p+1}{4}\sigma_{\min}(R)I,
	\end{align}
 uniformly for all $T\le k\le 2T-1$.
\end{lemma}
The above persistency of excitation condition holds uniformly over all times $k$ as long as $k\ge N_0$. This is why the burn-in time $N_0$ appears in~Theorem~\ref{OUT_THM_Main_Bound}; if $k$ is very small, then matrix $\sum_{j=p}^{k}Z_{j}Z^*_{j}$ is not even invertible.  A similar persistency of excitation result  was proved in~\cite{tsiamis2019finite} for a fixed time $k$. However, the result of Lemma~\ref{Lemma_PE} is more general since it holds for all $k$.

\textbf{Regularization and Truncation terms.} For the regularization term we follow the same steps as with the regression one. Since matrix $A-KC$ is stable, the truncation term decreases exponentially fast with $p$.  System quantity $\kappa$ governs how fast the observations grow polynomially. Parameter $\beta$ should be large enough cancel out this polynomial rate. This explains why $\kappa$ affects the choice of $\beta$ in~\eqref{OUT_EQN_beta}.

\textbf{Stable Systems.}
If $\rho(A)<1$, then we can exploit the fact that $Z_k$ converges exponentially fast to a stationary distribution. Hence the term $\bar{V}^{-1/2}_{k-1}Z_k$ will effectively be self-normalized, without the need to express $Z_k$ as a function of the past observations. 
In particular, for stable systems we prove a new and stronger persistency of excitation result. Denote:
\[
\Gamma_{Z,k}\triangleq \mathbb{E} Z_{k}Z_{k}^*
\]
Then, we have the following.
\begin{lemma}[Uniform Persistency of Excitation: Stable case]\label{Lemma_PE_stable}
	Consider the conditions of Theorem~\ref{STABLE_THM_Main_Bound}. Select a failure probability $\delta>0$. Let $T=2^{i-1}T_{\text{init}}$ for some fixed epoch $i$ with $p=\beta \log T$ the corresponding past horizon.
	There exists a $N_0=\mathrm{poly}(n,\beta,\log 1/\rho(A),\log 1/\delta)$ such that if $T\ge N_0$, with probability at least $1-\delta$:
	\begin{align}
	\sum_{j=p}^{k-1}Z_{j}Z^*_{j}\succeq \frac{k-p+1}{32}\Gamma_{Z,k},
	\end{align}
	uniformly for all $T\le k\le 2T-1$.
\end{lemma}

Hence, the term $\bar{V}^{-1/2}_{k-1}Z_k$ can be bounded by:
\[
\norm{\bar{V}^{-1/2}_{k-1}Z_k}_2\le O(\frac{1}{\sqrt{k-p+1}})\norm{\Gamma^{-1/2}_{Z,k}Z_k}_2
\]
where now the normalized term $\Gamma^{-1/2}_{Z,k}Z_k$ behaves like a standard isotropic Gaussian variable.
The term $\log 1/\rho(A)$ is due to the fact that it takes $O(\log 1/\rho(A))$ time for the state to approach the stationary distribution (mixing time).

%% file: Extensions.tex
\section{Extensions}\label{Section_Extensions}
In this section, we discuss generalizations of Algorithm~\ref{OUT_ALG_Output} and the regret analysis.
\paragraph{Alternative regret definition}
Denote the system responses by $g_{t}=C(A-KC)^{t-1}K$, for $t\ge 0$. Let $g\triangleq\set{g_t,t\ge 0}$ be the sequence of system responses. Then, a parameterization for online prediction could be
\[
\bar{y}^g_k=g_1y_{k-1}+\dots+g_{k}y_{0},\text{ for all }k\ge 0.
\]
Let $\mathcal{G}_{\rho,L}\triangleq\set{g:\:\snorm{g_t}_2\le L\rho^{t}}$ be the set of system responses which decay exponentially for some $L$ and $\rho<1$, which are larger than $M,\rho(A-KC)$ in Assumption~\ref{ASS_Bounded_Responses}. This set can include for example, stable IIR filters or FIR filters. 
Then, an alternative regret definition would be:
\begin{equation}\label{EXT_Online_Regret}
\tilde{\Reg}_N\triangleq 	\sum^{N}_{k=1}\snorm{y_k-\tilde{y}_k}^2-\inf_{g\in\mathcal{G}_{\rho,L}}\sum^{N}_{k=1}\snorm{y_k-\bar{y}^g_k}^2.
\end{equation}
The above definition captures the one in~\cite{kozdoba2019line}, where the unknown system lies in a finite set.

Since the observations increase at most polynomially fast and due to the properties of the Kalman filter, we can show that the difference $\tilde{\Reg}_N-\Reg_N$  depends on logarithmic terms of $N$. Hence our definition~\eqref{FOR_EQN_Regret_Output}, which does not require any model restriction is general. The details can be found in the Appendix.

\paragraph{$f$-steps ahead predictor}
An immediate generalization of Algorithm~\ref{OUT_ALG_Output} is to consider the $f-$steps ahead predictor, where $f$ is some future horizon. Instead of predicting only the next observation, we predict the sequence $y_{k}$, $\dots$, $y_{k+f-1}$. 
Denote the future observations and noises by:
\begin{align*}
Y_{k}&=\matr{{ccc}y_k^*&\cdots&y_{k+f-1}^*}^*\\
E^{+}_{k}&=\matr{{ccc}e_k^*&\cdots&e_{k+f-1}^*}^*.
\end{align*}
Similar to~\eqref{OUT_EQN_future_past_output}, we can establish a regression:
\begin{align*}
Y_k&=\O_f\K_p Z_k+\O_f(A-KC)^p\hat{x}_{k-p}+\T_f E^{+}_k
\end{align*}
where 
$
\K_p\triangleq\matr{{ccc}(A-KC)^{p-1}K&\cdots&K},
$ and
$\T_f$ is a lower triangular block Toeplitz matrix generated by $I,CK,\dots,CA^{f-2}K$.
The optimal Kalman filter predictor in this case is
\[
\hat{Y}_k=\O_f\K_p Z_k+\O_f(A-KC)^p\hat{x}_{k-p}
\]
Hence, the regret can be defined as in~\eqref{FOR_EQN_Regret_Output}, with the lowercase $y$ replaced with uppercase $Y$.
The online predictor~\eqref{OUT_EQN_Y_prediction} can be adapted here:
\[
\tilde{Y}_k=\tilde{G}_{k,f,p} Z_k,
\]
where $\tilde{G}_{k,f,p}$ is obtained similar to~\eqref{OUT_EQN_G_estimate} by regressing future observations $Y_t$ to past observations $Z_t$ from time $p$ up to $k-f$.
 The logarithmic regret guarantees of $\tilde{O}(1)$ also hold with the final bound depending polynomially on $f$ and $\norm{\T_f}_2$. 
\paragraph{State prediction}
If we have some knowledge about the state, e.g. the state space basis and the state space dimension $n$, then we can use the $f-$step ahead predictor to predict the hidden state $\hat{x}_k$. 
Notice that the Kalman filter state prediction $\hat{x}_k$ can be rewritten as:
\[
\hat{x}_k=\K_pZ_k+(A-KC)^p\hat{x}_{k-p}=\O^{\dagger}_f\hat{Y}_k
\]
If we know $\O_f$ and the future horizon is large enough $f\ge n$ we can compute the state prediction:
\[
\tilde{x}_k=\O^{\dagger}_f\tilde{Y}_k,
\]
where $\tilde{Y}_k$ is our $f-$step ahead prediction and $\dagger$ denotes the pseudo-inverse. In this case the regret:
\begin{equation}\label{EXT_EQN_State_Regret}
\Reg_{x,N}\triangleq \sum^{N}_{k=1}\snorm{x_k-\tilde{x}_k}^2-\sum^{N}_{k=1}\snorm{x_k-\hat{x}_k}^2
\end{equation}
will enjoy the same logarithmic guarantees.
Hence, our algorithm can be used to solve the adaptive Kalman filter problem posed in~\cite{mehra1970identification,anderson2005optimal}, where the dynamics $A,C$ are known but the noise statistics $Q,R$ are unknown, with logarithmic regret.

If we do not know $A,C$, then we could estimate the range space of $\O_f$ by performing singular value decomposition on $\tilde{G}_{k,f,p}$. However, there are infinite representations $\O_fS$, for any invertible $S$, all of which can explain the same observations. The definition~\eqref{EXT_EQN_State_Regret} is ill-posed since $x_k$ and $\tilde{x}_k$ might be based on different transformations $S$. Finding an alternative regret definition is subject of future work.
\paragraph{Logarithmic memory}
It is possible to achieve the logarithmic regret guarantees with logarithmic $O(\log N)$ memory, by modifying the initialization step in the beginning $T$ of every epoch $i$ in Algorithm~\ref{OUT_ALG_Output}. 
For stable systems, we could just reset $\tilde{G}_{T-1}$ and $\bar{V}_{T-1}$ to zero and $\lambda I$ respectively. 
This might not work for non-explosive systems, since $Z_{T,p}$ can be polynomially large in $T$. In this case, based on the regret analysis, we could initialize $\tilde{G}_{T-1},\bar{V}_{T-1}$ with the recent history $Z_{T-1,p},\dots,Z_{T-\hat{d},p}$ and $y_{T-1},\dots,y_{T-\hat{d}}$, where $\hat{d}$ is an upper bound for the degree of the minimal polynomial $a$. This gives us control over $\bar{V}_{T-1}^{-1/2}Z_{T,p}$--see Section~\ref{Section_Analysis}, and requires  only $O(\log T)$ memory.

%% file: Conclusion.tex
\section{Conclusion and Future Work}\label{Section_Discussion}
In this paper, we provided the first logarithmic regret upper bounds for learning the classical Kalman filter of an unknown system with unknown stochastic noise. Our regret analysis holds for non-explosive systems and our bounds do not degrade with the system stability gap.

Going forward, our paper opens up several research directions. An open question that is whether we can define an appropriate regret notion in the case of state prediction, when matrices $A,C$ are unknown, and prove logarithmic bounds. 
Another interesting direction is to study how the learning performance is affected by system theoretic properties, such as  the exponential quantity $d^{\kappa}$ in the case of systems with long chain structure, e.g.  $\kappa$-order integrators. 
Analyzing the regret of other online algorithms, e.g. extended least squares, is also an open problem.
Another challenging problem for both prediction and system identification is the case of explosive systems. Although in the fully observed case, this problem has been studied~\citep{faradonbeh2018finite,sarkar2018fast}, it remains open in the case of partially observable systems.
Finally, in this work we considered that the state is only driven by stochastic noise. A more general problem to study is when we also have exogenous inputs. One of the challenges is that it is harder to prove persistency of excitation in the case of closed-loop systems.

\section*{Acknowledgments}
The authors would like to thank Nikolai Matni for useful discussions.

%% file: Appendix_Regret.tex
\newpage
\tableofcontents
\newpage
\section*{Appendix}
\addcontentsline{toc}{part}{Appendix}
\section{Notation and organization of the Appendix}
	\paragraph{Structure.} In the first Sections~\ref{APP_Section_Lin_Sys},~\ref{APP_Section_Stat} we review results from system theory and statistics. These include the main tools with which we will prove Theorems~\ref{OUT_THM_Main_Bound},~\ref{STABLE_THM_Main_Bound}. In Section~\ref{APP_Section_Fixed}, we provide PAC bounds and persistency of excitation (PE) results for a fixed time $k$ (pointwise) and fixed past horizon $p$. In Section~\ref{APP_Section_Uniform}, we generalize those PAC bounds and PE results from pointwise to uniform over all times $k$ in one epoch. In Section~\ref{APP_Section_Normalized} we prove Lemma~\ref{OUT_LEM_self_normalization}. By combining the uniform bounds and Lemma~\ref{OUT_LEM_self_normalization}, we prove in Section~\ref{APP_Section_Epoch} that the square loss suffered within one epoch is logarithmic with respect the length of the epoch. Hence, we can now prove Theorem~\ref{OUT_THM_Main_Bound}--see Section~\ref{APP_Section_TH1}. In Section~\ref{APP_Section_Stable}, we analyze the case of stable systems and prove Theorem~\ref{STABLE_THM_Main_Bound}. Finally, in Section~\ref{APP_Section_ALT}, we show how the alternative regret definition~\eqref{EXT_Online_Regret} is equivalent to ours~\eqref{FOR_EQN_Regret_Output} up to logarithmic terms.
Section~\ref{APP_Section_Log} includes some technical results about logarithmic inequalities, which are used to show that the burn-in time $N_0$ depends polynomially on the various system parameters.
\paragraph{Notation.}
Before we proceed with the regret analysis, let us introduce some notation. A summary can be found in Table~\ref{FIXED_Table_notation}
We will analyze the performance of Algorithm~\ref{OUT_ALG_Output} based mainly on a fixed epoch $i$. Since the past horizon $p$ is kept constant during an epoch, we will drop the index $p$ from $Z_{k,p}$, $G_{p}$, $\tilde{G}_{k,p}$, $\bar{V}_{k,p}$ and write $Z_{k}$, $G$, $\tilde{G}_{k}$, $\bar{V}_{k}$ instead. 
Similar to the past outputs $Z_k$, we also define the past noises:
\begin{equation}\label{DEF_EQN_Past_Noises}
E_k\triangleq\matr{{ccc}e^*_{t-p}&\cdots&e^{*}_{t-1}}^*
\end{equation}
The batch past outputs, batch past noises, and batch past Kalman filter states are defined as:
\begin{equation}\label{DEF_EQN_Batch_vectors}
\bar{Z}_k\triangleq\matr{{ccc}Z_{p}&\cdots&Z_{k}},\, \bar{E}_k\triangleq\matr{{ccc}E_{p}&\cdots&E_{k}},\, 
\bar{X}_k\triangleq\matr{{ccc}\hat{x}_{0}&\cdots&\hat{x}_{k-p}}
\end{equation}
This notation will simplify the sums $\sum_{t=p}^{k} Z_tZ^*_t=\bar{Z}_k\bar{Z}^*_k$, $\sum_{t=p}^{k} E_tZ^*_t=\bar{E}_k\bar{Z}_k$ etc.

Recall the definition of the correlations between the current innovation and the past outputs $S_k\triangleq\sum_{t=p}^{k}e_{t}Z^*_t$ and the regularized autocorrelations of past outputs $\bar{V}_{k}\triangleq \lambda I+\bar{Z}_k\bar{Z}^*_k$.
The innovation sequence $e_k$ is i.i.d. zero mean Gaussian. Its covariance has a closed-form expression and is defined as:
\begin{equation}\label{DEF_EQN_Innovation_Covariance}
\bar{R}\triangleq \E e_ke^*_k=CPC^*+R,
\end{equation}
where $P$ is the solution to the Riccati equation~\eqref{FOR_EQN_Riccati}. 
Next we define the Toeplitz matrix $\T_k$, for some $k\ge 1$:
\begin{equation}\label{DEF_EQN_Toeplitz}
\T_k\triangleq \mathrm{Toep}(I,CK,\dots,CA^{k-2}K)=\matr{{cccc}I_m&0& &0\\CK&I_m&\cdots&0\\ \vdots&\vdots& &\vdots \\CA^{k-2}K&CA^{k-3}K&\cdots&I_m}.
\end{equation}

A useful property of system~\eqref{FOR_EQN_System_Innovation} is that the past outputs can be written as:
\begin{equation}\label{DEF_EQN_Past_Outputs_Toeplitz}
Z_t=\O_p\hat{x}_{t-p}+\T_pE_t
\end{equation}
The covariance of $\T_pE_t$ is denoted by:
\begin{equation}\label{DEF_EQN_Sigma_E}
\Sigma_E\triangleq \E \T_pE_tE^*_t\T^*_p=\T_p\diag{\bar{R},\dots,\bar{R}}\T^*_p.
\end{equation}

We define the covariance of the state predictions:
\begin{equation}\label{DEF_EQN_State_Covariance}
\Gamma_{k}\triangleq \E \hat{x}\hat{x}^*_k
\end{equation}
and the covariance of the past outputs:
\begin{equation}\label{DEF_EQN_Past_Outputs_Covariance}
\Gamma_{Z,k}\triangleq \E Z_kZ^*_k=\O_p\Gamma_{k-p}\O^*_p+\Sigma_E.
\end{equation}

Finally, let $A=SJS^{-1}$ be the Jordan form of $A$.
With the big  O notation we also hide parameters like $\norm{C}_2,\norm{K}_2$, $\norm{R}_2,\norm{S}_2,\norm{S^{-1}}_2$ etc.
\begin{table}[h]\caption{Notation table for fixed past horizon $p$}\label{FIXED_Table_notation}
	\begin{center}
		\begin{tabular}{r l c}
			\toprule
			$Z_{t}$&$\triangleq \matr{{ccc}y^*_{t-p}&\cdots&y^{*}_{t-1}}^*$& Past outputs at time $t$\\
			$E_{t}$&$\triangleq\matr{{ccc}e^*_{t-p}&\cdots&e^{*}_{t-1}}^*$& Past noises at time $t$\\
			$\bar{Z}_k$&$\triangleq\matr{{ccc}Z_{p}&\cdots&Z_{k}}$& Batch past outputs up to time $k$\\
			$\bar{E}_k$&$\triangleq\matr{{ccc}E_{p}&\cdots&E_{k}}$ &Batch past noises up to time $k$\\
			$\bar{X}_k$&$\triangleq\matr{{ccc}\hat{x}_{0}&\cdots&\hat{x}_{k-p}}$ &Batch past states up to time $k$\\
			$S_k$&$\triangleq \sum_{t=p}^{k}e_{t}Z^*_t$& Correlation of current noise with past outputs\\
			$V_k$&$\triangleq \bar{Z}_k\bar{Z}^*_k=\sum_{t=p}^{k}Z_{t}Z^*_t$& Gram matrix of past outputs\\
			$\bar{V}_k$&$\triangleq \lambda I+V_k$& Regularized Gram matrix of past outputs\\
			$\bar{R}$&$\triangleq\E e_ke_k^*$& Covariance of innovations\\
			$\T_k$&$\triangleq\mathrm{Toep}\paren{I,CK,\dots,CA^{k-2}K}$& See~\eqref{DEF_EQN_Toeplitz}, Toeplitz matrix of $I_m$ and system responses $CA^tK$\\
			$\Sigma_{E}$&$\triangleq\E \T_pE_t E_t^*T^*_p$ & Covariance of weighted past noises\\
			$\sigma_R$&$\triangleq\sigma_{\min}(R)$& Smallest singular value of $R$\\
			$\Gamma_{t}$&$\triangleq\E \hat{x}_t\hat{x}^*_t$ & Covariance of Kalman filter state prediction\\
			$\Gamma_{Z,t}$&$\triangleq\E Z_t Z_t^*$ & Covariance of past outputs\\ 
			$\tilde{G}_k$&$\triangleq (\sum_{t=p}^{k}y_tZ^*_t) \bar{V}^{-1}_k$& Estimated responses\\
			$G$&$\triangleq \matr{{ccc}C(A-KC)^{p-1}K&\dots&CK}$& Kalman filter responses\\
			\bottomrule
		\end{tabular}
	\end{center}
\end{table}

\section{Linear Systems Theory}\label{APP_Section_Lin_Sys}
\subsection{Bounds for system matrices}
Next, we provide a bound for the powers of $A$.
\begin{lemma}
	Consider matrix $A$ with Jordan form $SJS^{-1}$. Let $\kappa$ be the largest Jordan block of $A$ corresponding to a unit circle eigenvalue $\abs{\lambda}=1$. Let $\kappa_{\max}$ be the largest Jordan block among all eigenvalues.
	Then:
	\begin{enumerate}
				\item If the system is asymptotically stable $\rho(A)<1$, then $\norm{A^{i}}_2=O\paren{1}$,
		$S_t=\sum_{i=0}^{t}\norm{A^i}_2=O\paren{1}$, $\snorm{\T_t}_2=O(1)$, $\snorm{\O_t}_2=O(1)$, $\snorm{\Gamma_{t}}_2=O(1)$.
	\item If the system is non-explosive then $\norm{A^{i}}_2=O\paren{i^{\kappa-1}}$,
	$S_t=\sum_{i=0}^{t}\norm{A^i}_2=O\paren{t^{\kappa}}$, $\snorm{\O_t}_2=O\paren{t^{\kappa}}$, $\snorm{\T_t}_2=O\paren{t^{\kappa}}$ and $\snorm{\Gamma_{t}}_2=O\paren{t^{2\kappa-1}}$.
	\item For both cases
	\[
	\max_{0\le i\le d}\norm{A^i}_2\le O(d^{\kappa_{\max}-1})
	\]
\end{enumerate}
\end{lemma}

\begin{proof}
	
	\textbf{Proof of first part.}
	By Gelfand's formula~\cite{horn2012matrix}, for every $\epsilon>0$, there exists a $i_0=i_0(\epsilon)$ such that $\norm{A^i}\le (\rho(A)+\epsilon)^i$, for all $i\ge i_0$.
	Just pick $\epsilon$ such that $\rho\paren{A}+\epsilon<1$. Then,
	\[
	S_t\le \sum_{i=0}^{i_0} \norm{A^i}_2+\frac{1}{1-\rho(A)-\epsilon}=O\paren{1}.
	\] 
	The proof for the other system quantities is similar.
	
	\textbf{Proof of second part.}
	Assume that $A$ is equal to a $n\times n$ Jordan block corresponding to $\lambda=1$. The proof for the other cases is similar.
	Then we have that:
	\[
	A^{i}=\matr{{cccc}1& {i }\choose {1}&\dots& {i} \choose{ n-1}\\0&1&\dots&{i} \choose{ n-2}\\ & &\ddots &\\0&0&\dots&1}
	\]
	By Lemma~\ref{ALG_LEM_TOEPLITZ}, we obtain:
	\[
	\norm{A^{i}}_2 \le \sum_{k=0}^{n-1} {{i} \choose{ k}}\le \paren{\frac{e i}{n-1}}^{n-1}
	\]
	where the second inequality is classical, see Exercise~0.0.5 in~\cite{vershynin2018high}.
	
	Hence, we have:
	\[
	S_t\le t \paren{\frac{e t}{n-1}}^{n-1}=O\paren{t^n}
	\]
		\[
	\norm{\O_t}_2\le \norm{C}_2 S_{t-1}=O\paren{t^n}.
	\]
	
	By Lemma~\ref{ALG_LEM_TOEPLITZ}
		\[
	\norm{\T_t}_2\le 1+\norm{C}_2\norm{K}_2 S_{t-2}=O\paren{t^n}.
	\]
	Finally
	\[
	\norm{\Gamma_t}_2\le \norm{K\bar{R}K^*}_2\sum_{i=0}^{t-1}\norm{A^i}^2=tO\paren{t^{2n-2}}.
	\]
	
		\textbf{Proof of third part.} In the general case of pairs $(\lambda_j,\kappa_j)$ of eigenvalues and Jordan block sizes, similar to the previous proof:
		\[
		\max_{0\le i\le d}\norm{A^i}_2\le \norm{S}_2\norm{S^{-1}}_2 \max_{j} \abs{\lambda_j^{d-\kappa_j+1}}\paren{\frac{e d}{\kappa_j-1}}^{\kappa_j-1}= O(d^{\kappa_{\max}-1})
		\]
\end{proof}

In the above proofs we used a standard result for the norm of (block) Toeplitz matrices. A proof can be found in~\cite{tsiamis2019finite}.
\begin{lemma}[Toeplitz norm]\label{ALG_LEM_TOEPLITZ}
	Let $M\in \R^{m_1n\times m_2n}$, for some integers $n,m_1,m_2$ be an (upper) block triangular Toeplitz matrix:
	\[
	M=\matr{{cccccc}M_1&M_2&M_3&\cdots&\cdots&M_n\\0&M_1&M_2& & &M_{n-1}\\\vdots & & \ddots & \ddots& &\vdots\\ \\ \\ & & & &M_1&M_2\\0&0&\cdots& &0&M_1},
	\]
	where $M_{i}\in \R^{m_1\times m_2}$, $i=1,\dots,n$.
	Then:
	\[
	\norm{M}_2\le \sum_{i=1}^{n}\norm{M_i}_2
	\]
\end{lemma}
\subsection{Properties of covariance matrix $\Gamma_{k}$}
The following result is standard, for a proof see~\cite{tsiamis2019finite}.
\begin{lemma}[monotonicity]
Consider system~\eqref{FOR_EQN_System_Innovation}, with $\Gamma_k\triangleq \E \hat{x}_k\hat{x}^*_k$. The sequence $\Gamma_k$ is increasing in the positive semi-definite cone.
\end{lemma}

The following results is also standard, but we include a proof for completeness.
\begin{lemma}[Lyapunov difference equation]\label{SYS_LEM_Covariance}
	Consider system~\eqref{FOR_EQN_System_Innovation}, with $\Gamma_k\triangleq \E \hat{x}_k\hat{x}^*_k$.  Assume that the system is stable with $\rho(A)<1$. Then, the sequence $\Gamma_k$ satisfies:
	\[
	\Gamma_k=A\Gamma_{k-1}A^*+K\bar{R}K^*
	\]
	and converges to $\Gamma_{\infty}\succ 0$, the unique positive definite solution of the Lyapunov equation:
	\[
		\Gamma_{\infty}=A\Gamma_{\infty}A^*+K\bar{R}K^*.
	\]
	Moreover, there exists a $\tau=\frac{1}{\log 1/\rho(A)}\tilde{O}(\max\set{\log\mathrm{cond}(\Gamma_{\infty}),\kappa_{\max}})$ such that
	\[
	\Gamma_{k}\succeq \frac{1}{2}\Gamma_{\infty}, \text{ for all }k\ge \tau
	\]
	where $\mathrm{cond}(\Gamma_{\infty})=\frac{\sigma_{\min}(\Gamma_{\infty})}{\sigma_{\max}(\Gamma_{\infty})}$.
\end{lemma}
\begin{proof}
Since $A$ is stable $\Gamma_{\infty}=\sum_{k=0}^{\infty} A^{k}K\bar{R}K^*(A^*)^k$ is well defined and solves the Lyapunov equation. 
Since $(A,K)$ is controllable $\Gamma_{\infty}$ is strictly positive definite:
\[
\Gamma_{\infty}\succeq \matr{{cccc}KR^{1/2}&AKR^{1/2}&\dots A^{n-1}KR^{1/2}}\matr{{cccc}KR^{1/2}&AKR^{1/2}&\dots A^{n-1}KR^{1/2}}^*\succ 0,
\]
where the controllability matrix $\matr{{cccc}KR^{1/2}&AKR^{1/2}&\dots A^{n-1}KR^{1/2}}$ has full rank. It is unique since the operator $\mathcal{L}(M)=M-AMA^*$ is invertible; it has eigenvalues bounded below by $1-\rho^2(A)$.

Notice that $\Gamma_0=0\preceq \Gamma_{\infty}$ and by induction, we can show that $\Gamma_k\preceq \Gamma_{\infty}$. Since $\Gamma_k$ is also monotone, it converges to the unique $\Gamma_{\infty}$.

Now form the difference
\[
\Gamma_k-\Gamma_{\infty}=-A^{k}\Gamma_{\infty}(A^*)^k.
\]
It is sufficient to find a $\tau$ such that:
\[
 \norm{A^{\tau}}^2 \sigma_{\max}(\Gamma_{\infty}) \le \frac{\sigma_{\min}(\Gamma_{\infty})}{2}
\]
Since the norm of grows as fast as $ \norm{A^{\tau}}_2=O(\rho(A)^{\tau-\kappa_{\max}+1}\tau^{\kappa_{\max}})$, it is sufficient to pick:
\[
\tau\ge \frac{\kappa_{\max}\log \tau}{\log\frac{1}{\rho(A)}}-\frac{\log \mathrm{cond}(\Gamma_{\infty})/2}{\log\frac{1}{\rho(A)}}+\kappa_{\max}-1.
\]
By Lemma~\ref{LogIn_Lem_Log}, the order of $\tau$ is
\[
\tau=\frac{1}{\log 1/\rho(A)}\tilde{O}(\max\set{\log\mathrm{cond}(\Gamma_{\infty}),\kappa_{\max}})
\]
\end{proof}
\begin{lemma}[Convergence rate: Lyapunov difference equation]
	Consider system~\eqref{FOR_EQN_System_Innovation}, with $\Gamma_k\triangleq \E \hat{x}_k\hat{x}^*_k$.  Assume that the system is stable with $\rho(A)<1$. 
\end{lemma}
\begin{proof}

\end{proof}
\subsection{Proof of Lemma~\ref{OUT_LEM_ARMA_representation}}
We first prove a result for just the observations $y_k$.
\begin{proposition}[ARMA-like representation]\label{ARMA_PROP_ARMA_form}
	Consider system~\eqref{FOR_EQN_System_Innovation} with $a(s)=s^{d}-a_{d-1}s^{d-1}-\dots-a_0$ the minimal polynomial with degree $d$.
	Then the outputs can be represented by the ARMA recursion:
	\begin{equation}\label{ARMA_EQN_System}
	y_{k}=a_{d-1}y_{k-1}+\dots+a_0 y_{k-d}+\sum_{s=0}^{d}M_{s}e_{k-s},
	\end{equation}
	where $M_s$ are matrices of appropriate dimensions such that:
	\[
	\norm{M_s}_2\le \norm{a}_1\max\set{\norm{C}_2\norm{K}_2\max_{0\le i\le d}\norm{A^{i-1}}_2,1},
	\]
	where $\norm{a}_1$ denotes the $\ell_1$ norm of the polynomial coefficients $1+\sum_{i=0}^{d-1} \abs{a_i}$.
\end{proposition}
\begin{proof}
	We start from the fact that:
	\[
	y_{k-t}=CA^{d-t}\hat{x}_{k-d}+\underbrace{\sum_{s=t+1}^{d}CA^{s-t-1}Ke_{k-s}+e_{k-t}}_{\tilde{e}_{k,t}},
	\]
	for $t=0,\dots,d$.
From the properties of the minimal polynomial we obtain that:
	\[
	CA^d\hat{x}_k=a_{d-1}CA^{d-1}\hat{x}_{k-1}+\dots+a_0C\hat{x}_{k-d}
	\]
	which leads to:
	\begin{align}
	y_{k}&=a_{d-1}y_{k-1}+\dots+a_0y_{k-d}+\tilde{e}_{k,0}-\sum_{t=1}^{d}a_{d-t}\tilde{e}_{k,t}\\
	&a_{d-1}y_{k-1}+\dots+a_0y_{k-d}+\sum_{s=0}^{d}M_se_{k-s},
	\end{align}
	with $M_0=I$ and
	\[
	M_s=-a_{d-s}-\sum_{t=1}^{s} a_{d-t} CA^{t-1}K+CA^{s-1}K
	\]
\end{proof}
The same will now hold for the past outputs:
\[
Z_{k}=a_{d-1}Z_{k-1}+\dots+a_0Z_{k-p}+\underbrace{\sum_{s=0}^{d}\diag{M_s,\dots,M_s}E_{k-s}
}_{\delta_k}\]
where
\[
E_k\triangleq\matr{{ccc}e_{k-p}^*&\cdots&e_{k-1}^{*}}^*.
\]
We can bound $\delta_k$ by:
\[
\norm{\delta_k}_2\le (d+1)\max_{0\le s\le d}{\norm{M_s}_2}\max_{s\le k}\norm{E_{s}}_2\le (d+1)\max_{0\le s\le d}{\norm{M_s}_2}\sqrt{p}\max_{s\le k}\norm{e_{s}}_{2}.
\]
Define:
\begin{equation}\label{ARMA_EQN_Delta}
\Delta\triangleq (d+1)\norm{a}_1\max\set{\norm{C}_2\norm{K}_2\max_{0\le i\le d}\norm{A^{i-1}}_2,1}\sqrt{p}
\end{equation}
The fact that $\Delta=O(d^{\kappa_{\max}-1}\norm{a}_1\sqrt{p})$ follows from the fact that $\norm{A^{i-1}}_2=O(i^{\kappa_{\max}-2})$.\hfill$\square$ 
\section{Statistical Toolbox}\label{APP_Section_Stat}
\subsection{Least singular value of Toeplitz matrix}
Let $u_t\in\R^{m}$, $t=0,\dots$ be an i.i.d. sequence, where $u_k\sim\mathcal{N}(0,I)$ are isotropic Gaussians. The following results shows that the Toeplitz matrix is well conditioned with high probability.
Similar results have been reported in~\cite{sarkar2019finite,oymak2018non}. 
Compared to~\cite{oymak2018non} we have better dependence between the number of samples $k$ and $\log 1/\delta$, which will allow us to prove uniform persistency of excitation. 
 Compared to~\cite{sarkar2019finite}, we have similar terms, but we also include universal constants.
\begin{lemma}\label{STAT_LEM_Toeplitz}
	Let $u_t\in \R^{m}$, $t=0,\dots,$ be an i.i.d. sequence of Gaussian variables with unit covariance matrix.
	Consider the Toeplitz matrix \[U=\matr{{cccc}u_{k-p} & u_{k-p-1}&\dots&u_{0}\\u_{k-p+1}& u_{k-p}&\dots&u_{1}\\\vdots\\u_{k-1}&u_{k-2}&\dots&u_{p-1}}.\] If
	\begin{equation*}
	k\ge f_1\paren{p,\delta}\triangleq p+128\paren{mp^2\log 9+p\log2+p\log\tfrac{1}{\delta}}
	\end{equation*}
	then with probability at least $1-\delta$:
	\begin{equation}\label{NoisePE_Basic_Inequality}
	\tfrac{1}{2}(k-p+1)I\preceq UU^*\preceq \tfrac{3}{2}(k-p+1)I.
	\end{equation}
\end{lemma}
\begin{proof}
	Consider a $1/4$-net $\N$ of the unit sphere $\S^{mp-1}$. Then, from~\citep[Exercise 4.4.3 b]{vershynin2018high}:
	\begin{equation}\label{NoisePE_EQN_Proof_Net}
	\norm{UU^*-(k-p+1)I}_2\le 2\sup_{v\in \N}\abs{v^*UU^*v-(k-p+1)}.
	\end{equation}
	Denote the partition of $v$ in $p$ blocks of length $m$ as:
	\[
	v^*=\matr{{cccc}v^*_1&v^*_2&\dots&v^*_p}
	\]
	Notice that $U^*v\in\R^{k-p}$ is zero-mean Gaussian with covariance matrix:
	\begin{equation}\label{NoisePE_EQN_Toeplitz}
	\Sigma=\matr{{ccccccc}1& r(1)&\cdots&r(p-1)&0&\cdots&0\\
	r(1)&1&\cdots&r(p-2)&r(p-1)&\cdots&0\\ \vdots& &\ddots \\r(p-1)&r(p-2)&\cdots&1&r(1)&\cdots&0\\0&r(p-1)&\cdots&r(1)&1&\cdots&0\\& &\vdots\\ 0& 0&\cdots &0 & 0& \cdots&1}
	\end{equation}
	where we define the convolutions:
	\[
	r(t)=\sum_{s=1}^{p-t}v^*_sv_{t+s}, t=0,\dots,p-1
	\]
	with $r(0)=1$.
	Due to the Toeplitz structure:
	\[
	\norm{\Sigma}_2\le 1+2\sum_{i=1}^{p-1}\abs{r(i)}\le 1+  \sum_{i,j=1,i\neq j}^{p}\norm{v}_i\norm{v}_j=(\sum_{i=1}^p \norm{v_i})^2,
	\]
	where the first inequality follows from Lemma~\ref{ALG_LEM_TOEPLITZ}, while the second follows from the triangle inequality.
	By Cauchy-Schwartz:
	\[
	\norm{\Sigma}_2\le (\sqrt{p}\norm{v}_2)^2=p
	\]
	Moreover, for positive definite matrices the following bound holds
	\[
	\norm{\Sigma}^2_F\le \Tr{\Sigma}\norm{\Sigma}_2=(k-p+1)p.
	\]
	Hence, by Lemma~\ref{NoisePE_Scalar_Concentration} below, we obtain
	\[
	\P(\abs{v^*UU^*v-(k-p+1)}\ge \frac{k-p+1}{4})\le 2e^{-\frac{k-p+1}{128 p}}.
	\]
	Taking the union bound over the whole net $\N$, we obtain
	\[
	\P(\sup_{v\in N}\abs{v^*UU^*v-(k-p)}\ge \frac{k-p+1}{4})\le 2*\abs{\mathcal{N}}e^{-\frac{k-p+1}{128 p}}\le 2*9^{mp}e^{-\frac{k-p+1}{128 p}},
	\]
	where we used that $\abs{\mathcal{N}}\le 9^{mp}$~\cite{vershynin2018high}.
	From equation~\eqref{NoisePE_EQN_Proof_Net}, we get that:
	\[
	\norm{UU^*-(k-p)I}_2\le \frac{k-p+1}{2}
	\]
	with probability at least $1-\delta$
	if we choose
	\[
	k-p+1\ge 128\paren{mp^2\log 9+p\log2+p\log\tfrac{1}{\delta}},
	\] 
	which completes the proof. 
\end{proof}

The following result is standard but we include it to have a sense of the universal constants. It is a specialized version of the Hanson-Wright theorem~\cite{vershynin2018high}. 
\begin{lemma}[Hanson-Wright specialization]\label{NoisePE_Scalar_Concentration}
	Let $Z\in\R^N\sim\N\paren{0,\Sigma}$. Then, for every $t\ge 0$ the following inequalities hold:
	\begin{align*}
	&	\P\paren{Z^*Z\ge\Tr\Sigma+t}\le \exp\paren{- \min\set{\frac{t^2}{8\snorm{\Sigma}^2_{F}},\frac{t}{8\snorm{\Sigma}_2}}}\\
	&	\P\paren{-Z^*Z\ge-\Tr\Sigma+t}\le \exp\paren{- \frac{t^2}{8\snorm{\Sigma}^2_{F}}}
	\end{align*}
\end{lemma}
\begin{proof}
	We only prove the first inequality, the second one is similar and easier. Following the standard exponential bound procedure:
	\[
	\P\paren{Z^*Z\ge\Tr\Sigma+t}\le e^{-\tfrac{1}{2}s(\Tr\Sigma+t)}\E\paren{e^{\frac{s}{2}Z^*Z}},\text{ for all }0\le s\le \tfrac{1}{2\snorm{\Sigma}_2}
	\]
	By Lemma~\ref{NoisePE_Lemma_Moment_Generating}, we obtain:
	\[
	\psi(t)\triangleq	\log\P\paren{Z^*Z\ge\Tr\Sigma+t}\le -\tfrac{1}{2}\paren{s\Tr\Sigma+st+\log\det(I-s\Sigma)}
	\]
	Now let the eigenvalues of $\Sigma$ be $\lambda_i$, for $i=1,\dots,N$. Then:
	\[
	\psi(t)=-\tfrac{1}{2}\paren{st+\sum_{i=1}^N s\lambda_i+\log(1-s\lambda_i)}.
	\]
	Hence, by Lemma~\ref{NoisePE_logarithm_inequality}
	\[
	\psi(t)\le \tfrac{1}{2}\paren{-st+s^2\norm{\Sigma}^2_{F}},
	\]
	where all the products $0\le s\lambda_i\le 1/2$ since $\Sigma$ is positive definite and $s\le \tfrac{1}{2\snorm{\Sigma}_2}$ . The result follows by minimizing over $s$, from Lemma~\ref{NoisePE_Minimum}.	
\end{proof}

\begin{lemma}\label{NoisePE_Lemma_Moment_Generating}
	Let $Z\in\R^N\sim\N\paren{0,\Sigma}$ and $0\le s < 1/(2\snorm{\Sigma}_2)$. Then:
	\[
	\E(e^{\frac{s}{2}Z^*Z})=\frac{1}{\sqrt{\det\paren{I-s\Sigma}}}
	\]
\end{lemma}
\begin{proof}
	First, notice that for $0\le s < 1/(2\snorm{\Sigma}_2)$, the eigenvalues of $\paren{I-s\Sigma}$ are bounded away from $1/2$, hence:
	\[
	\Sigma^{-1}-sI \succeq \frac{1}{2\snorm{\Sigma}_2}I \succ 0
	\] $I-s\Sigma$ is invertible.
	Now let $P=(\Sigma^{-1}-sI)^{-1} $. By changing the measure:
	\[
	\E(e^{\frac{s}{2}Z^*Z})=\int \frac{(2\pi)^{-N/2}}{\sqrt{\det\Sigma}}e^{-\frac{1}{2}z^*P^{-1}z}dz=\frac{\sqrt{\det P}}{\sqrt{\det\Sigma}}=\frac{1}{\sqrt{\det\paren{I-s\Sigma}}}
	\]
\end{proof}
\begin{lemma}\label{NoisePE_logarithm_inequality}
	Let $0\le x\le 1/2$, then the following inequality holds
	\[
	x+\log{(1-x)}\ge -x^2,
	\]	
	where $\log$ is the natural logarithm.
\end{lemma}
\begin{proof}
	Let $g(x)=x+\log{(1-x)}+x^2$. Then:
	\[
	g'(x)=1-\frac{1}{1-x}+2x \ge 0, 
	\]	
	for $0\le x\le 1/2$. Hence, $g(x)\ge g(0)=0$
\end{proof}
\begin{lemma}\label{NoisePE_Minimum}
	Let $c,\bar{s}>0$, then
	\[
	\min_{0\le s\le \bar{s}} -as+cs^2\le \max\set{-\frac{a^2}{4c},-\frac{a\bar{s}}{2}}
	\]
\end{lemma}
\begin{proof}
	By elementary calculus:
	\[
	s^*=\left\{ 
	\begin{aligned}
	&\frac{a}{2c},&&\text{ when }\frac{a}{2c}\le \bar{s}\\
	&\bar{s}, &&\text{ when }\frac{a}{2c}> \bar{s}
	\end{aligned}\right.
	\]
	As a result, the minimum is equal to
	\[
	\left\{ 
	\begin{aligned}
	&-\frac{a^2}{4c},&&\text{ when }\frac{a}{2c}\le \bar{s}\\
	&-a\bar{s}+c\bar{s}^2\le -a\bar{s}/2, &&\text{ when }\frac{a}{2c}> \bar{s}
	\end{aligned}\right.
	\]
\end{proof}
\subsection{Self-normalized martingales}
The following theorem, which can be found in~\cite{tsiamis2019finite} is an extension of Theorem~1 in~\cite{abbasi2011improved} and Proposition~8.2 in~\cite{sarkar2018fast}.
\begin{theorem}[Cross terms,\cite{tsiamis2019finite}]\label{MART_THM_Vector}
	Let $\set{\F_t}_{t=0}^{\infty}$ be a filtration. Let $\eta_{t}\in\R^m$, $t\ge 0$ be  $\F_t$-measurable, independent of $\F_{t-1}$. Suppose also that $\eta_{t}\sim\mathcal{N}(0,I)$ is isotropic Gaussian.
	Let $X_{t}\in\R^{d}$, $t\ge 0$ be $\F_{t-1}-$measurable. Assume that $V$ is a $d\times d$ positive definite matrix. For any $t\ge 0$, define:
	\[
	\bar{V}_t=V+\sum_{s=1}^{t}X_sX_s^*,\qquad S_t=\sum_{s=1}^{t} X_sH^*_s,
	\]
	where
	\[
	H^{*}_s=\matr{{ccc}\eta^*_s&\dots&\eta^*_{s+r-1}}\in\R^{rm},
	\]
	for some integer $r$.
	Then, for any $\delta>0$, with probability at least $1-\delta$, for all $t\ge 0$
	\begin{equation*}
	\norm{\bar{V}_t^{-1/2} S_t }^2_2\le 8r\paren{\log\frac{r5^m}{\delta}+\frac{1}{2}\log\det\bar{V}_tV^{-1}}
	\end{equation*}
	\hfill $\diamond$
\end{theorem}
\subsection{Gaussian suprema}
\begin{lemma}\label{SUP_LEM_Supremum}
Consider $v_t\in \R^d\sim \N\paren{0,I}$ i.i.d., for $t=1,\dots,k$. Let $X_k\in \R^{q}$ be a linear combination:
\[
X_k\triangleq\sum_{t=1}^{k} M_{k,t}v_t, \text{ for }k=1,\dots,T
\]
where $M_{t,k}\in \R^{q\times d}$. For some $\mu>0$ define:
\[
\Sigma_{k}\triangleq \mu I+\mathbb{E}X_kX^*_k
\]
Fix a failure probability $\delta>0$.
Then with probability at least $1-\delta$:
\begin{equation}\label{BOUND_EQN_Supremum}
\sup_{k=1,\dots,T}{\snorm{\Sigma^{-1/2}_kX_k}_2}\le \sqrt{q}+\sqrt{2\log\frac{ T}{\delta}}
\end{equation}
If $\mathbb{E}X_kX^*_k$ is invertible for all $k=1,\dots,T$, the result holds for $\mu=0$.
\end{lemma}
\begin{proof}
Fix a $k$. By the following Lemma~\ref{SUP_LEM_Norm} we obtain that we probability at least $1-\delta/T$:
\[
\snorm{\Sigma^{-1/2}_kX_k}_2\le \sqrt{q}+\sqrt{2\log\frac{ T}{\delta}}.
\]
The result follows by a simple union bound.
\end{proof}
\begin{lemma}\label{SUP_LEM_Norm}
Consider $v_t\in \R^d\sim \N\paren{0,I}$ i.i.d., for $t=1,\dots,k$. Let $X_k\in \R^{q}$ be a linear combination:
\[
X_k\triangleq\sum_{t=1}^{k} M_tv_t,
\]
where $M_t\in \R^{q\times d}$. For some$\mu>0$ define:
\[
\Sigma_{k}\triangleq \mu I+\mathbb{E}X_kX^*_k
\]
Then:
\begin{equation}\label{BOUND_EQN_Normalized_Norm}
P(\snorm{\Sigma^{-1/2}_kX_k}_2> \sqrt{q}+t)\le e^{-t^2/2}
\end{equation}
If $\mathbb{E}X_kX^*_k$ is invertible the result holds for $\mu=0$.
\end{lemma}
\begin{proof}
An application of Jensen's inequality gives:
	\[
	\E\snorm{\Sigma^{-1/2}_kX_k}_2\le \sqrt{\E X^*_k\Sigma^{-1}_kX_k}=\sqrt{\Tr \Sigma^{-1}\mathbb{E}X_kX^*_k}\le \sqrt{q}
	\]
	Meanwhile,
	\[
	\norm{\Sigma^{-1/2}_{k}\matr{{ccc}M_1&\cdots&M_k}}_2\le 1
	\]
	since by definition $\Sigma_{k}\succeq \sum_{t=1}^{k}M_tM^*_t$.
	Hence, the function $\snorm{\Sigma^{-1/2}_kX_k}_2$ is Lipschitz with respect to $v_{t,i}$, for $t=1,\dots,k$, $i=1,\dots,d$ with Lipschitz constant $1$.
By concentration of Lipschitz functions of independent Gaussian variables~\citep[Theorem 5.6]{boucheron2013concentration}:
	\[
P(\snorm{\Sigma^{-1/2}_kX_k}_2> \sqrt{q}+t)\le e^{-t^2/2}
	\]
\end{proof}

\section{PAC bounds and persistency of excitation for fixed-time and fixed-past}\label{APP_Section_Fixed}
In this section, we include results for persistence of excitation and for identification of the system parameters for a fixed time instance $k$, fixed confidence $\delta$ and a fixed past horizon $p$. To avoid excess notation, we drop the dependence on $p$ and $k$ in this section.
\begin{theorem}[PAC bounds for identification]\label{FIXED_THM_Identification}
	Consider system~\eqref{FOR_EQN_System_Innovation} with observations $y_{0},\dots,y_k$. Fix a past horizon $p$ and consider the notation of Table~\ref{FIXED_Table_notation}. 
	Define
	\begin{equation}\label{FIXED_EQN_Index_Functions}
	\begin{aligned}
	k_1(p,\delta)&\triangleq p+128\paren{mp^2\log 9+p\log2+p\log\tfrac{1}{\delta}}\\
	k_2(k,p,\delta)&\triangleq p+\frac{64}{\min\set{4,\sigma_R}}  \paren{4pn\log\paren{\frac{n\norm{\O_p}_2^2\norm{\Gamma_{k-p}}_2}{\delta}+1}+8p\log \frac{p5^m}{\delta}}
	\end{aligned}
	\end{equation}
With probability at least $1-5\delta$ the following events hold at the same time:
	\begin{enumerate}[wide, labelwidth=!, labelindent=0pt]
		\item[a) Persistency of excitation] 
		\begin{equation}	
\mathcal{E}_{PE}\triangleq\set{	
	\begin{aligned} \T_p\bar{E}_k\bar{E}_k^*\T_p^*&\succeq \frac{k-p+1}{2}\Sigma_{E}\succeq \frac{k-p+1}{2}\sigma_RI,\\
	\bar{Z}_k\bar{Z}^*_k&\succeq	\frac{1}{2}\O_p\bar{X}_k\bar{X}_k^*\O_p^* +\frac{1}{2}\T_p\bar{E}_k\bar{E}_k^*\T_p^*  ,	
		\end{aligned}}
		\end{equation}
	if $k$ satisfies the following perstistency of excitation requirement
		\begin{equation}	\label{FIXED_EQN_Condition_OutputPE}
		k\ge \max\set{k_1(p,\delta),k_2(k,p,\delta)}.
		\end{equation}
		\item[ b) Upper bound for outputs]
		\begin{equation}
		\mathcal{E}_{\bar{Z}}\triangleq\set{\bar{Z}_k\bar{Z}^*_k \preceq  (k-p+1)\frac{mp}{\delta}\Gamma_{Z,k}} \label{FIXED_EQN_Output_Upper}	
	\end{equation}
		\item[c) Cross term upper bound]
		\begin{align}
		\mathcal{E}_{\mathrm{cross}}&\triangleq\set{
	\norm{S_k\bar{V}_k^{-1/2}}_{2}\le g_1\paren{k,p,\delta} \label{FIXED_EQN_Cross_Terms_Bound}},\text{ where }\\
	g_1\paren{k,p,\delta}&\triangleq\sqrt{8}\sqrt{\norm{\bar{R}}mp} \sqrt{\log\frac{3mp}{\delta}+\frac{1}{2}\log(k-p+1)+\frac{1}{2mp}\log\det \paren{\Gamma_{Z,k}\lambda^{-1}+I}}\label{FIXED_EQN_Cross_Terms_g}
		\end{align}
	\end{enumerate}
\end{theorem}
\begin{proof}
Define the primary events:
					\begin{subequations}
		\begin{alignat}{2}
\mathcal{E}_{\bar{X}}\triangleq&\set{\bar{X}_k\bar{X}^*_k \preceq  (k-p+1)\frac{n}{\delta}\Gamma_{k-p}}\label{FIXED_EQN_State_Upper}\\
	\mathcal{E}_E\triangleq&\set{ \T_p\bar{E}_k\bar{E}_k^*\T_p^*\succeq \frac{k-p+1}{2}\Sigma_E}\\
		\mathcal{E}_{XE}\triangleq&\set{ \norm{\bar{W}_k^{-1/2}\bar{X}_k\bar{E}^*_k\Sigma^{-1/2}_E}^2_2\le 8p\paren{\log \frac{p5^{m}}{\delta}+\frac{1}{2}\log\det\bar{W}_kW^{-1}}},\label{FIXED_EQN_Cross_State_Noise}\\
			\mathcal{E}_{EZ}\triangleq&\set{ \norm{\bar{R}^{-1/2}S_k\bar{V}_k^{-1/2}}^2_2\le 8\paren{\log \frac{5^{m}}{\delta}+\frac{1}{2}\log\det\bar{V}_kV^{-1}}}\label{FIXED_EQN_Cross_Output_Noise}
		\end{alignat}
		\end{subequations}
		where matrices $\bar{W}_t,\bar{L}_t,\bar{B}_t$ are appropriate Gram matrices that normalize the correlations:
			\begin{align}
&\bar{W}_k\triangleq \bar{X}_k\bar{X}^*_k+W,&&W\triangleq 		\frac{k-p+1}{\norm{\O_p}^2_2}I\\
&\bar{V}_k\triangleq \bar{Z}_k\bar{Z}^*_k+V,&&V\triangleq\lambda I
	\end{align}
We will show that $\mathcal{E}_{\bar{Z}}$ and all of the above events occur with probability at least $1-\delta$ each. 
Moreover  \[\mathcal{E}_{PE}\cap\mathcal{E}_{\mathrm{cross}}\supseteq \mathcal{E}_{\bar{X}}\cap\mathcal{E}_{\bar{Z}}\cap\mathcal{E}_{E}\cap\mathcal{E}_{XE}\cap\mathcal{E}_{ZE}\]
Hence by a union bound:
\[
\P(\mathcal{E}_{PE}\cap\mathcal{E}_{\bar{Z}}\cap\mathcal{E}_{\mathrm{cross}})\ge 1-5\delta.
\]

	\textbf{Part a: All primary events occur with probability at least $1-\delta$.}
	
	\noindent The fact that $\P\paren{\mathcal{E}_{\bar{X}}}\ge 1-\delta$, $\P\paren{\mathcal{E}_{\bar{Z}}}\ge 1-\delta$ follows by a Markov inequality argument--see~\cite{simchowitz2018learning}. The fact that $\P(\mathcal{E}_{E})\ge 1-\delta$ follows from Lemma~\ref{NoisePE_LEM_Noise_PE}. For the remaining events, we apply Theorem~\ref{MART_THM_Vector}. Notice that $\bar{R}^{-1/2}e_k$ and $\Sigma_E^{-1/2}\E_k$ are isotropic so that the conditions of Theorem~\ref{MART_THM_Vector} hold.
	
	\textbf{Part b: Event $\mathcal{E}_{PE}$}
	
	\noindent 	From Lemma~\ref{OutputPE_LEM_PE} below, we have that $\mathcal{E}_{PE}\supseteq \mathcal{E}_{\bar{X}}\cap\mathcal{E}_E\cap\mathcal{E}_{XE}$ if $k$ satisfies~\eqref{FIXED_EQN_Condition_OutputPE}.
	
	\textbf{Part c: Event $\mathcal{E}_{\mathrm{cross}}$}
	
		\noindent We show that $\mathcal{E}_{\mathrm{cross}}\supseteq \mathcal{E}_{\mathrm{\bar{Z}}}\cap\mathcal{E}_{\mathrm{ZE}}$. 		
		Conditioned on $\mathcal{E}_{\mathrm{\bar{Z}}}\cap\mathcal{E}_{\mathrm{ZE}}$,  we have
		\begin{align*}
		\norm{S_kV^{-1/2}_k}^2_2&\le 8\norm{\bar{R}}_2\paren{\log \frac{5^{m}}{\delta}+\frac{1}{2}\log\det\bar{V}_kV^{-1}} \\
		&\le 8\norm{\bar{R}}_2\paren{\log \frac{5^{m}}{\delta}+\frac{1}{2}\log\det\paren{({k-p+1)\frac{mp}{\delta}\Gamma_{Z,k}\lambda^{-1}+I}}}	\\
		&\le 8\norm{\bar{R}}_2\paren{\log \frac{5^{m}}{\delta}+\frac{mp}{2}\log\frac{mp}{\delta}+\frac{mp}{2}\log(k-p+1)+\frac{1}{2}\log\det\paren{\Gamma_{Z,k}\lambda^{-1}+\frac{\delta}{mp(k-p+1)}I}}\\
		&\le g_1^2(k,p,\delta)
		\end{align*}
		where we used the simplification $\frac{\delta}{mp(k-p+1)}<1$ and 
		$\log\frac{5^m}{\delta}\le \frac{mp}{2}\log\frac{3mp}{\delta}$, for $p\ge 2$.
\end{proof}
\subsection{Persistency of excitation proofs}
First, we prove persistence of excitation of the past noises in finite time. 
\begin{lemma}[Noise PE]\label{NoisePE_LEM_Noise_PE}
Consider the conditions of Theorem~\ref{FIXED_THM_Identification}. If
		\begin{equation}\label{NoisePE_EQN_Basic_Condition}
	k\ge p+128\paren{mp^2\log 9+p\log2+p\log\tfrac{1}{\delta}}
	\end{equation}
	then with probability at least $1-\delta$
	\begin{equation*}
	\frac{k-p+1}{2} \sigma_{R}I \preceq	\frac{k-p+1}{2} \Sigma_{E} \preceq \T_p \bar{E}_k \bar{E}_k^* \T^*_p\preceq \frac{3(k-p+1)}{2} \Sigma_{E}.
	\end{equation*}
\hfill $\diamond$
\end{lemma}
\begin{proof}
Notice that $U_k\triangleq\Sigma^{-1/2}_E\T_pE_k$ satisfy the conditions of Lemma~\ref{STAT_LEM_Toeplitz}.
Hence, under condition~\eqref{NoisePE_EQN_Basic_Condition}, with probability at least $1-\delta$:
	\[
		\frac{k-p+1}{2} I \preceq \sum_{t=p}^{k}U_t U_t^* \preceq \frac{3(k-p+1)}{2} I.
	\]
	Multiplying from both sides with $\Sigma_E^{1/2}$ gives \[
\frac{k-p+1}{2} \Sigma_{E} \preceq \T_p \bar{E}_k \bar{E}_k^* \T^*_p\preceq \frac{3(k-p+1)}{2} \Sigma_{E}	
	\]
Finally, from~\cite{tsiamis2019finite}, we have $\Sigma_{E}\succeq \sigma_{R}I$.
\end{proof}
Next, we prove persistency of excitation for the past outputs.
\begin{lemma}[Output PE]\label{OutputPE_LEM_PE}
Consider the conditions of Theorem~\ref{FIXED_THM_Identification} and the definition of $\mathcal{E}_{E}$, $\mathcal{E}_{\bar{X}}$, $\mathcal{E}_{XE}$. If:
\begin{align}\label{OutputPE_EQN_OutputPE_Cond}
k\ge p-1+\frac{64}{\min\set{4,\sigma_R}}  \paren{4pn\log\paren{\frac{n\norm{\O_p}_2^2\norm{\Gamma_{k-p}}_2}{\delta}+1}+8p\log \frac{p5^m}{\delta}}
\end{align}
then
\[
\set{\bar{Z}_k\bar{Z}_k^*\succeq \frac{1}{2}\O_p\bar{X}_k\bar{X}_k^*\O_p+\frac{1}{2}\T_p\bar{E}_k\bar{E}_k\T_p}\supseteq \mathcal{E}_{E}\cap\mathcal{E}_{\bar{X}}\cap\mathcal{E}_{XE}
\]
\hfill $\diamond$
\end{lemma}
\begin{proof}
The past outputs can be written as:
	\[
	\bar{Z}_k=\O_p\bar{X}_k+\T_p\bar{E}_k.
	\]
	For simplicity, we rewrite $\Sigma^{-1/2}_E\T_p\bar{E}_k=\bar{U}_k$, where $\bar{U}_k$ is defined similarly to $\bar{E}_k$ but has unit variance components.
	As a result the sample-covariance matrix will be:
	\[
	\frac{1}{k-p+1}\bar{Z}_k\bar{Z}^*_k=\frac{1}{k-p+1}\paren{O_p\bar{X}_k\bar{X}^*_k\O^*_p+\Sigma^{1/2}_E\bar{U}_k\bar{U}^*_k\Sigma^{1/2}_E+\O_p\bar{X}_k\bar{U}^*_k\Sigma^{1/2}_E+\Sigma^{1/2}_E\bar{U}_k\bar{X}^*_k\O^*_p}
	\]
The proof proceeds in two steps. First, we bound the cross-terms based on the events $\mathcal{E}_{\bar{X}},\mathcal{E}_{XE}$. Second, we show that if $k-p+1$ is large enough, then the cross-terms are small enough so that
\[
\frac{1}{k-p}\paren{\O_p\bar{X}_k\bar{U}^*_k\Sigma^{1/2}_E+\Sigma^{1/2}_E\bar{U}_k\bar{X}^*_k\O^*_p}\preceq \frac{1}{2}\frac{1}{k-p}\paren{O_p\bar{X}_k\bar{X}^*_k\O^*_p+\Sigma^{1/2}_E\bar{U}_k\bar{U}^*_k\Sigma^{1/2}_E}
\]

\noindent	\textbf{Cross-term bounds}

\noindent 
Conditioned on $\mathcal{E}_{\bar{X}}$
	\begin{align*}
	\log\det\bar{W}_kW^{-1} &\le \log\det \paren{\frac{n\norm{\O_p}^2_2 }{\delta}\Gamma_{k-p}+I}\\
	&=\log\paren{\frac{n\norm{\O_p}^2_2\norm{\Gamma_{k-p}}_2}{\delta}+1}^{n}=n\log\paren{\frac{n\norm{\O_p}^2_2\norm{\Gamma_{k-p}}_2}{\delta}+1}.
	\end{align*}
	Conditioned also on $\mathcal{E}_{XE}$:
	\[
	\norm{\bar{W}_k^{-1/2}\bar{X}_k\bar{U}^*_k}^2_2\le 8p\paren{\log \frac{p5^{m}}{\delta}+\frac{1}{2}n\log\paren{\frac{n\norm{\O_p}^2_2\norm{\Gamma_{k-p}}_2}{\delta}+1}}
	\]
	For simplicity denote:
	\[
	\C_{XE}\triangleq  \sqrt{8p\paren{\log \frac{p5^{m}}{\delta}+\frac{1}{2}n\log\paren{\frac{n\norm{\O_p}^2_2\norm{\Gamma_{k-p}}_2}{\delta}+1}}}.
	\]

Let now $u\in\R^{mp}$, $\norm{u}_2=1$ be an arbitrary unit vector. Then, consider the quadratic form \begin{align*}
\frac{1}{k-p}\paren{u^*\O_p\bar{X}_k\bar{U}^*_k\Sigma_E^{1/2}u+u^*\Sigma^{1/2}_E\bar{U}_k\bar{X}^*_k\O^*_pu}.
\end{align*}
Conditioned on $\set{\norm{\bar{W}_k^{-1/2} \bar{X}_k\bar{U}^*_k }^2_2 	\le \C_{XE}}\cap \mathcal{E}_E\cap \mathcal{E}_X$, we can bound the cross terms by:
\begin{align*}
\frac{1}{k-p+1}\norm{u^*\O_p\bar{X}_k\bar{U}^*_k\Sigma_E^{1/2}u+u^*\Sigma^{1/2}_E\bar{U}_k\bar{X}^*_k\O^*_pu}_2&\le \frac{2}{k-p}\norm{u^*\O_p \bar{W}^{1/2}_k\bar{W}_k^{-1/2}\bar{X}_k\bar{U}^*_k}_2\norm{\Sigma^{1/2}_Eu}_2\\
&\le \frac{2}{k-p+1}\norm{u^*\O_p \bar{W}^{1/2}_k}_2\norm{\bar{W}_k^{-1/2}\bar{X}_k\bar{U}^*_k}_2\norm{\Sigma^{1/2}_Eu}_2\\
&\le 2\sqrt{\frac{1}{k-p+1}u^*\O_p\bar{X}_k\bar{X}^*_k\O^*_pu+1}\frac{\C_{XE}}{\sqrt{k-p+1}}\norm{\Sigma^{1/2}_Eu}_2
\end{align*}

\noindent\textbf{Cross-terms are dominated}

\noindent To complete the proof, it is sufficient to show that if~\eqref{OutputPE_EQN_OutputPE_Cond} holds then also 
	\begin{align*}
2\sqrt{\frac{1}{k-p+1}u^*\O_p\bar{X}_k\bar{X}^*_k\O^*_pu+1}\frac{\C_{XE}}{\sqrt{k-p+1}}\norm{\Sigma^{1/2}_Eu}_2\le \frac{1}{2}\frac{1}{k-p+1}\paren{u^*\O_p\bar{X}_k\bar{X}^*_k\O^*_pu+u^*\Sigma^{1/2}_E\bar{U}_k\bar{U}^*_k\Sigma^{1/2}_Eu},
	\end{align*}
	for any unit vector $u$.
	Define:
	\begin{align*}
	a&=\frac{1}{k-p+1}u^*\O_p\bar{X}_k\bar{X}^*_k\O^*_pu\\
	b&=\frac{1}{k-p+1}u^*\Sigma_E u
	\end{align*}
 Notice that on $\mathcal{E}_E$ we have $\Sigma^{1/2}_E\bar{U}_k\bar{U}^*_k\Sigma^{1/2}_E\succeq \frac{k-p+1}{2}\Sigma_E$. Thus, it is sufficient to show
 	\begin{align*}
 2\sqrt{a+1}\frac{\C_{XE}}{\sqrt{k-p+1}}\sqrt{b}\le \frac{a}{2}+\frac{b}{4}.
 \end{align*}
 To complete the proof, we apply the following Lemma~\ref{OutputPE_LEM_Elementary_Minimum}, where we exploit the fact that $b\ge \sigma_{\min}\paren{\Sigma_E}\ge \sigma_R$. It follows that it is sufficient to have:
 \[
 \C_{XE}/\sqrt{k-p+1}\le \frac{2,\sqrt{\sigma_R}}{8}
 \]
\end{proof}
\begin{lemma}[Elementary Minimum]\label{OutputPE_LEM_Elementary_Minimum}
	Let $a \ge 0$ and $b\ge \sigma_R>0$. Then if
	\[
	\gamma\le \frac{\min\set{2,\sqrt{\sigma_R}}}{8}
	\]
	the following inequality is true:
	\[
\frac{a}{2}+\frac{b}{4}-2\sqrt{a+1}\sqrt{b}\gamma \ge 0
	\]
\end{lemma}
\begin{proof}
	Define the function $f(a,b)=\frac{a}{2}+\frac{b}{4}-2\sqrt{a+1}\sqrt{b}\gamma$. By optimizing over $a$, we obtain that the minimum over $a$ is:
	\[
	\min_{0\le a}f(a,b)=\left. \begin{aligned}
	&\frac{b}{4}-2\sqrt{b}\gamma,&&\text{ if }2\gamma\sqrt{b}\le 1\\
	&b\paren{\frac{1}{4}-2\gamma^2}-\frac{1}{2},&&\text{ if }2\gamma\sqrt{b}> 1
	\end{aligned}\right\}.
	\]
	The condition $\gamma\le \frac{\min\set{2,\sqrt{\sigma_R}}}{8}$ guarantees that \[\min_{0\le  a,\sigma_R\le b\le 4}f(a,b)=\frac{b}{4}-2\sqrt{b}\gamma\ge \frac{b-\sqrt{b}\sigma_R}{4}\ge 0 .\]
	For $b>4$:
	\[
	\min_{0\le a,4< b}f(a,b)=b\paren{\frac{1}{4}-2\gamma^2}-\frac{1}{2}\ge b\paren{\frac{1}{4}-\frac{1}{8}}-\frac{1}{2}=\frac{b-4}{8} \ge 0 
	\]
\end{proof}
\section{Proof of Lemma~\ref{Lemma_PE}}\label{APP_Section_Uniform}
It follows from the lemma below, which is more general. 
\begin{lemma}[Uniform PAC bounds]\label{UniformPE_LEM}
		Consider the conditions of Theorem~\ref{OUT_THM_Main_Bound}. Select a failure probability $\delta>0$. Let $T=2^{i-1}T_{\text{init}}$ for some fixed epoch $i$ with $p=\beta \log T$ the corresponding past horizon. Consider the definition of $g_1(k,p,\delta)$ in~\eqref{FIXED_EQN_Cross_Terms_g}. There exists a $N_0=\mathrm{poly}(n,\beta,\kappa,\log 1/\delta)$ such that with probability at least  $1-5\sum_{k=T}^{2T-1}\frac{1}{k^2}\delta$ the following events hold:
	\begin{equation}\label{UniformPE_EQN_Bounds}
	\mathcal{E}_{\mathrm{unif}}\triangleq \set{	\begin{aligned}
	\sum_{j=p}^{k}Z_{j}Z^*_{j}&\preceq  (k-p+1)\frac{ k^2 mp}{\delta}\Gamma_{Z,k}\\
	\norm{S_k\bar{V}_k^{-1/2}}_2&\le g_1(k,p,\delta/k^2)
	\end{aligned}, \text{ for all }T\le k\le 2T-1}
	\end{equation}
	\begin{equation}\label{UniformPE_EQN_PE}
	\mathcal{E}^{\mathrm{PE}}_{\mathrm{unif}}\triangleq \set{	\sum_{j=p}^{k}Z_{j}Z^*_{j}\succeq \frac{k-p+1}{4}\sigma_{R}I,\text{ for all }\max\set{N_0,T} \le k\le 2T-1 }
	\end{equation}
\end{lemma} 
\begin{proof}
	Consider $k_1(p,\delta)$, $k_2(k,p,\delta)$ defined in~\eqref{FIXED_EQN_Index_Functions} and
	define:
	\begin{equation}\label{UniformPE_EQN_N0}
	N_0\triangleq \min\set{t:\: k\ge k_1(\beta\log k,\delta/k^2), k\ge k_2(k,\beta \log k, \delta/k^2),\text{ for all }k\ge t}.
	\end{equation}
	The dominant terms in $k_1,k_2$ increase with at most the order of $\log^2 k$:
	\begin{align*}
	&	\max\set{k_{1}(\beta\log k,\delta/k^2),k_2(k,\beta \log k,\delta/k^2)}\\
	&=O(m\beta^2 \log^2 k+n\beta\log k\log \norm{\Gamma_{k-p}}_2+\beta n \log^2 k+\beta n\log k\log \frac{1}{\delta}),
	\end{align*}
	while $\norm{\Gamma_{k-p}}_2=O(k^{2\kappa-1})$. By the technical Lemmas~\ref{LogIn_Lem_Log},~\ref{LogIn_Lem_LogSq} it follows that $N_0$ depends polynomially on the arguments $\beta,n,m,\log 1/\delta,\kappa$.
	By the definition of $N_0$, if $k\ge \max\set{N_0,T}$ then:
	\[
		k\ge \max\set{k_{1}(\beta\log k,\delta/k^2),k_2(k,\beta \log k,\delta/k^2)}\ge\max\set{k_{1}(p,\delta/k^2),k_2(k,p,\delta/k^2)}
	\]
	
	Now, fix a $k$ such that $T\le k\le 2T-1$. By Theorem~\ref{FIXED_THM_Identification}, with probability at least $1-5\delta/k^2$ we have:
	\begin{align*}
\sum_{j=p}^{k}Z_{j}Z^*_{j}&\preceq  (k-p+1)\frac{ k^2 mp}{\delta}\Gamma_{Z,k}\\
\norm{S_k\bar{V}_k^{-1/2}}_2&\le g_1(k,p,\delta/k^2)\\
	\sum_{j=p}^{k}Z_{j}Z^*_{j}&\succeq \frac{k-p+1}{4}\sigma_{R}I,\text{ if }k\ge N_0
	\end{align*}
The uniform result follows by a union bound over all $k$ in $T,\dots,2T-1$.
\end{proof}
\section{Proof of Lemma~\ref{OUT_LEM_self_normalization}}\label{APP_Section_Normalized}
We prove a slightly more general version.
\begin{lemma}
	Fix a $p$ and consider the notation of Table~\ref{FIXED_Table_notation}. Consider an $i\ge 0$. The following inequality is true:
	\begin{equation}\label{EPOCH_EQN_prediction_term}
	\sum_{k=T}^{2T-1}Z^*_{k-i} \bar{V}^{-1}_k Z_{k-i}\le \log\det (\bar{V}_{2T-i-1}\bar{V}^{-1}_{T-i-1})
	\end{equation}	
\end{lemma}
\begin{proof}
	Since $\bar{V}_k$ is increasing in the positive semidefinite cone:
	\[
	\sum_{k=T}^{2T-1}Z^*_{k-i} \bar{V}^{-1}_k Z_{k-i}\le  \sum_{k=T}^{2T-1}Z^*_{k-i} \bar{V}^{-1}_{k-i} Z_{k-i}
	\]
	Hence, it is sufficient to prove the inequality for $i=0$.
	Recall that $\bar{V}_{k}=\bar{V}_{k-1}+Z_{k}Z^*_{k}$.
	Consider the identity:
	\[\det{\bar{V}_{k}}=\det\paren{\bar{V}_{k-1}+Z_kZ^{*}_k}=\det{\bar{V}_{k-1}}\det\paren{I+\bar{V}^{-1/2}_{k-1}Z_kZ^*_k\bar{V}^{-1/2}_{k-1}}=\det{\bar{V}_{k-1}}\paren{1+Z^*_{k}\bar{V}^{-1}_kZ_k}\]
	where the last equality follows from the identity $\det(I+FB)=\det(I+BF)$. Rearranging the terms gives:
	\[
	Z^*_{k}\bar{V}^{-1}_kZ_k=1-\frac{\det \bar{V}_{k-1}}{\det \bar{V}_{k}}\le \log\det \bar{V}_{k}-\log\det \bar{V}_{k-1},
	\]
	where the inequality follows from the fact that
	the sequence $\bar{V}_k$ is increasing in the positive semidefinite cone and the elementary inequality:
	\[
	1-x\le \log 1/x,\text{ for }x\le1.
	\]
	Since the upper bound telescopes, we finally get
	\[
	\sum_{k=T}^{2T-1} Z^*_k \bar{V}^{-1}_k Z_k \le \sum_{T}^{2T-1} \log\det \bar{V}_{k}-\log\det \bar{V}_{k-1}=\log\det \bar{V}_{2T-1}-\log\det \bar{V}_{T-1}
	\]
\end{proof}
\section{Analysis within one epoch}\label{APP_Section_Epoch}
We will analyze the $\ell_2$ square loss for the duration of one epoch, from time $T$ up to time $2T-1$ with fixed past horizon $p=\beta \log T$. We have three cases: i) persistency of excitation is  established $T\ge N_0$, where $N_0$ is defined in~\eqref{UniformPE_EQN_N0}; ii) persistency of excitation is not established $T< N_0$; iii) warm-up epoch from $0$ to $T_{\text{init}}-1$.

For now, we concentrate on the first two cases.
Consider the $\ell_2$ loss within the epoch:
\begin{equation}\label{EPOCH_EQN_Loss_Partial}
\LL_{T}^{2T-1}\triangleq\sum_{k=T}^{2T-1}\norm{\hat{y}_k-\tilde{y}_k}^2_2.
\end{equation}
Based on the notation of Table~\ref{FIXED_Table_notation}, the error between the Kalman filter prediction and our online algorithm is:
\begin{align*}	
\tilde{y}_{k}-\hat{y}_{k}=&
\underbrace{S_{k-1}\bar{V}^{-1}_{k-1}Z_{k}}_{\text{regression}}+\underbrace{\lambda G\bar{V}^{-1}_{k-1}Z_{k}}_{\text{regularization}}
+\underbrace{C(A-KC)^p\paren{\sum_{i=T}^{k-1}\hat{x}_{i-p}Z^*_i\bar{V}^{-1}_{k-1}Z_k-\hat{x}_{k-p}}}_{\text{truncation bias}}\\
&=S_{k-1}\bar{V}^{-1}_{k-1}Z_{k}+\lambda G\bar{V}^{-1}_{k-1}Z_{k}
+C(A-KC)^p\paren{\bar{X}_{k-1}\bar{Z}_{k-1}\bar{V}^{-1}_{k-1} Z_k-\hat{x}_{k-p}},
\end{align*}
with the notation of Table~\ref{FIXED_Table_notation}.
By Cauchy-Schwarz,
the submultiplicative property of norm and by the fact that $\norm{\bar{Z}_{k-1}\bar{V}^{-1/2}_{k-1}}_2\le 1$:
\begin{align}
\norm{\tilde{y}_{k}-\hat{y}_{k}}^2_2&\le 4\sup_{T\le t\le 2T-1}\paren{\norm{S_{t-1}\bar{V}^{-1/2}_{t-1}}^2_2+\norm{\lambda G\bar{V}^{-1/2}_t}_2^2+\norm{C(A-KC)^p}^2_2\norm{\bar{X}_{t-1}}^2_2}\norm{\bar{V}^{-1/2}_{k-1}Z_k}^2_2\nonumber\\\label{EPOCH_EQN_Square_Loss_Terms}
&+ 4\norm{C(A-KC)^p}^2_{2}\norm{\hat{x}_{k-p}}^2_2.
\end{align} 
Hence bounding the $\LL^{2T-1}_T$ consists of three steps, bounding the supremum 
\[
\sup_{T\le t\le 2T-1}\paren{\norm{S_{t-1}\bar{V}^{-1/2}_{t-1}}^2_2+\norm{\lambda G\bar{V}^{-1/2}_t}_2^2+\norm{C(A-KC)^p}^2_2\norm{\bar{X}_{t-1}}^2_2}
\]
the sum
\[
\sum_{k=T}^{2T-1} \norm{\bar{V}^{-1/2}_{k-1}Z_k}^2_2,
\]
and the sum
\[
4\norm{C(A-KC)^p}^2_{2}\sum_{k=T}^{2T-1} \norm{\hat{x}_{k-p}}^2_2,
\]
This is what we do in the following theorem.
\begin{theorem}[Square loss within epoch]\label{EPOCH_THM_main}
		Consider the conditions of Theorem~\ref{OUT_THM_Main_Bound}. Let $a$ be the minimal polynomial of $A$ with degree $d$,  $\Delta$ defined as in~\eqref{ARMA_EQN_Delta}. Fix two failure probabilities $\delta,\delta_1>0$  and consider $N_0$ defined as in~\eqref{UniformPE_EQN_N0} based on $\delta$. Let $T=2^{i-1}T_{\text{init}}$ for some fixed epoch $i\ge 1$ with $p=\beta \log T$ the corresponding past horizon. Then, with probability at least $1-5\sum_{k=T}^{2T-1}\frac{1}{k^2} \delta-2\delta_1$:
	\begin{equation}\label{EPOCH_EQN_Loss_noPE}
\LL^{2T-1}_T\le \mathrm{poly}(\Delta,\norm{a}^2_2,n,\beta,\kappa,\log 1/\delta,\log 1/\delta_1)\paren{\tilde{O}(T)+\tilde{O}(\rho(A-KC)^pT^{2\kappa+1})}.
	\end{equation}
	If moreover $T\ge N_0$ then also:
		\begin{equation}\label{EPOCH_EQN_Loss_PE}
	\LL^{2T-1}_T\le \mathrm{poly}(\Delta,\norm{a}^2_2,n,\beta,\kappa,\log 1/\delta,\log 1/\delta_1)\paren{\tilde{O}(1)+\tilde{O}(\rho(A-KC)^pT^{2\kappa})}.
	\end{equation}
\end{theorem}
\begin{proof}
	\textbf{Uniform events occur with high probability. } Consider the primary events $\mathcal{E}_{\mathrm{unif}}$ and $\mathcal{E}^{\mathrm{PE}}_{\mathrm{unif}}$ defined in~\eqref{UniformPE_EQN_Bounds},~\eqref{UniformPE_EQN_PE} and:
	\begin{align}
	\mathcal{E}_x&=\set{\sup_{k\le 2T-1}\norm{\Gamma^{-1/2}_{k}\hat{x}_k}^2_2\le \paren{\sqrt{n}+\sqrt{2\log \frac{4T}{\delta_1}}}^2}\\
	\mathcal{E}_e&=\set{\sup_{k\le 2T-1}\norm{\bar{R}^{-1/2}e_k}^2_2\le \paren{\sqrt{m}+\sqrt{2\log\frac{2T}{\delta_1}}}^2}
	\end{align}
	Based on Lemma~\ref{UniformPE_LEM}, Lemma~\ref{SUP_LEM_Supremum}, and a union bound the all events $\mathcal{E}_{\mathrm{unif}}\cap \mathcal{E}^{\mathrm{PE}}_{\mathrm{unif}}\cap\mathcal{E}_x\cap\mathcal{E}_e$ occur with probability at least $1-5\sum_{k=T}^{2T-1}\frac{1}{k^2}\delta-\delta_1$. Now, we can bound all terms of the square loss based on the above events.
	
\noindent	\textbf{Bound on $\norm{S_{k-1}\bar{V}^{-1/2}_{k-1}}^2_2$.}
	From the definition of $\mathcal{E}_{\mathrm{unif}}$:
	\[
\sup_{T\le k\le 2T-1}	\norm{S_{k-1}\bar{V}^{-1/2}_{k-1}}^2_2\le g^2(k,p,\delta/k^2)=\mathrm{poly}(n,\beta,\kappa,\log 1/\delta)\tilde{O}(1)
	\]

\noindent	\textbf{Bound on $\norm{\lambda G\bar{V}^{-1/2}_t}^2_2$.}	
	We simply have: $\norm{\lambda G\bar{V}^{-1/2}_t}^2_2\le \lambda \norm{G}^2_2$

\noindent	\textbf{Bound on $\sup_{k\le 2T}\norm{\hat{x}_k}^2_2$}
	
	Since the covariances $\Gamma_{k}$ are increasing:
	\[
	\sup_{k\le 2T-1}\norm{\hat{x}_k}^2_2\le \norm{\Gamma_{2T-1}}_2\sup_{k\le 2T-1}\norm{\Gamma^{-1/2}_{k}\hat{x}_k}^2_2.
	\]

\noindent		\textbf{Bound on $\norm{C(A-KC)^p}^2_2\norm{\bar{X}_{t-1}}^2_2$.}
		
	Notice that $\norm{\bar{X}_{t-1}}^2_2\le 2T\sup_{k\le 2T-1}\norm{\hat{x}_k}^2_2$. From the bound above:
\[
	\norm{C(A-KC)^p}^2_2\norm{\bar{X}_{t-1}}^2_2\le \mathrm{poly}(n,\log 1/\delta_1)\tilde{O}(\rho(A-KC)^{p}T^{2\kappa})
	\]
	since $\norm{\Gamma_{2T-1}}_2=O(T^{2\kappa-1})$.
	
\noindent		\textbf{Bound on $\norm{C(A-KC)^p}^2_{2}\sum_{k=T}^{2T-1} \norm{\hat{x}_{k-p}}^2_2$.}
	
	It is similar to the previous step since:
	\[
\sum_{k=T}^{2T-1} \norm{\hat{x}_{k-p}}^2_2\le T \sup_{k\le 2T-1}\norm{\hat{x}_k}^2_2
	\]

\noindent \textbf{Bound on the sum of $\norm{\bar{V}^{-1/2}_{k-1} Z_k}^2_2$.}
By Lemma~\ref{EPOCH_LEM_Normalized}:
\[
	\sum_{k=T}^{2T-1}\norm{\bar{V}^{-1/2}_{k-1} Z_k}^2_2\le 2d\norm{a}^2_2 \log\det(\bar{V}_{2T-1}\lambda^{-1})+2\Delta\sup_{k\le 2T}\norm{e_k}^2_2 \sum_{k=T}^{2T-1}\norm{\bar{V}^{-1/2}_{k-1}}^2_2
\]
He have two cases:
 \begin{align*}
\sum_{k=T}^{2T-1}\norm{\bar{V}^{-1/2}_{k-1}}^2_2&\le \frac{1}{\sigma_R} \sum_{T}^{2T-1}\frac{1}{k-p}\le \frac{1}{\sigma_R} \log\frac{2T-p-1}{T-p-1},\text{ if }T\ge N_0\\
 \sum_{k=T}^{2T-1}\norm{\bar{V}^{-1/2}_{k-1}}^2_2&\le \frac{T}{\lambda},\text{ if }T< N_0.
 \end{align*}
 The terms $\log\det(\bar{V}_{2T-1}\lambda^{-1})$, $\sup_{k\le 2T}\norm{e_k}^2_2$ can be bounded based on $\mathcal{E}_{\mathrm{unif}},\mathcal{E}_{e}$.
\end{proof}

\subsection{Bounding the sum of $\norm{\bar{V}^{-1/2}_{k-1} Z_k}^2_2$}
In the following lemma, we bound term $\norm{\bar{V}^{-1/2}_{k-1} Z_k}_2$, which the key to obtaining bounds in the non-explosive regime. We will apply this in both cases i), ii).
\begin{lemma}[Normalized matrix $\bar{V}^{-1/2}_{k-1} Z_k$]\label{EPOCH_LEM_Normalized}
	Consider the conditions of Theorem~\ref{OUT_THM_Main_Bound}. Let $a$ be the minimal polynomial of $A$ with degree $d$ and $\Delta$ defined as in~\eqref{ARMA_EQN_Delta}. Let $T=2^{i-1}T_{\text{init}}$ for some fixed epoch $i$ with $p=\beta \log T$ the corresponding past horizon. Then,
	\begin{equation}
	\sum_{k=T}^{2T-1}\norm{\bar{V}^{-1/2}_{k-1} Z_k}^2_2\le 2d\norm{a}^2_2 \log\det(\bar{V}_{2T}\lambda^{-1})+2\Delta\sup_{k\le 2T}\norm{e_k}^2_2 \sum_{k=T}^{2T-1}\norm{\bar{V}^{-1/2}_{k-1}}^2_2.
	\end{equation}
\end{lemma}
\begin{proof}
	We replace $Z_k=a_{d-1}Z_{k-1}+\dots+a_0 Z_{k-d}+\delta_k$. Then by two applications of Cauchy-Schwarz:
	\[
	\norm{\bar{V}^{-1/2}_{k-1} Z_k}^2_2\le 2(a^2_{d-1}+\dots+a^2_0)\sum_{i=0}^{d-1} Z_{k-i}\bar{V}^{-1}_{k-1}Z_{k-i}+2 \norm{\bar{V}^{-1/2}_{k-1}\delta_k}^2_2.
	\]
	Now by Lemma~\ref{OUT_LEM_self_normalization} and Lemma~\ref{OUT_LEM_ARMA_representation} it follows that:
	\begin{equation*}
	\sum_{k=T}^{2T-1}\norm{\bar{V}^{-1/2}_{k-1} Z_k}^2_2\le 2\norm{a}^2_2 \sum_{i=0}^{d-1}\log\det(\bar{V}_{2T-i}\lambda^{-1})+2\Delta \sup_{k\le 2T}\norm{e_k}_2 \sum_{k=T}^{2T-1}\norm{\bar{V}^{-1/2}_{k-1}}^2_2.
	\end{equation*}
	The result follows from the fact that the sequence $\bar{V}_{2T-i}$ is monotone.
\end{proof}

\subsection{Warm-up epoch}
\begin{lemma}[Warm-up epoch]\label{EPOCH_LEM_Warm_up}
	Consider the conditions of Theorem~\ref{OUT_THM_Main_Bound} and the $\ell_2$ loss $\LL_{T_{\text{init}}}=\sum_{t=0}^{T_{\text{init}}}\norm{\hat{y}_t-\tilde{y}_t}^2_2$. Then with probability at least $1-\delta$:
	\[
	\LL_{T_{\text{init}}}\le \mathrm{poly}(m,\log1/\delta)T_{\text{init}}\norm{C\Gamma_{T_{\text{init}}}C^*+\bar{R}}_2\sqrt{\log T_{\text{init}}}=\left.\begin{aligned}&\tilde{O}(T^{2\kappa}_{\text{init}}), \text{ if }\kappa\ge 1
\\ &\tilde{O}(T_{\text{init}}), \text{ if }\kappa=0
	\end{aligned}\right\}\]
\end{lemma}
\begin{proof}
	During this time, we have $\tilde{y}_k=0$. Let $\Gamma_{y,k}=C\Gamma_{k}C^*+\bar{R}$. Then,
	\[
	\LL_{T_{\text{init}}}\le T_{\text{init}} \sup_{k\le T_{\text{init}}}\norm{\hat{y}_k}^2_2 \le T_{\text{init}}\norm{\Gamma_{y,T_{\text{init}}}}_2 \sup_{k\le T_{\text{init}}}\norm{\Gamma^{-1/2}_{y,T_{\text{init}}}\hat{y}_k}^2_2.
	\]
	By Lemma~\ref{SUP_LEM_Supremum} and by monotonicity of $\Gamma_{k}$, with probability at least $1-\delta$:
	\[
	\LL_{T_{\text{init}}}\le T_{\text{init}}\norm{\Gamma_{y,T_{\text{init}}}}_2 \paren{\sqrt{m}+\sqrt{2\log\frac{T_{\text{init}}}{\delta}}}=\left.\begin{aligned}&\tilde{O}(T^{2\kappa}_{\text{init}}), \text{ if }\kappa\ge 1
	\\ &\tilde{O}(T_{\text{init}}), \text{ if }\kappa=0	\end{aligned}\right\}
	\]
\end{proof}

\section{Proof of Theorem~\ref{OUT_THM_Main_Bound}}\label{APP_Section_TH1}
Recall that the regret can be decomposed in two terms:
\[
\Reg_N= \LL_N+2\sum_{k=0}^{N}{e^*_k\paren{\hat{y}_k-\tilde{y}_k}}
\]
where $\LL_N$ is the square loss and the other term is a martingale.

\textbf{Square loss bound.}
Without loss of generality assume that $N=2T_i-1=T_{\text{init}}2^{i}$ is the end of an epoch, where $i$ is the total number of epochs. The number of epochs $i$ depends logarithmically on $N$. Then the square loss $\LL_N$ can be written as:
\[
\LL_N=\LL_{T_{\text{init}}}+\sum_{j=1}^{i-1}\LL^{2T_j-1}_{T_{j}}.
\]
Let $N_0$ be defined as in~\eqref{UniformPE_EQN_N0}. Select
\begin{equation}\label{OUT_EQN_beta_choice}
\beta\ge \frac{\kappa}{\log 1/\rho(A-KC)}
\end{equation}
Then by Theorem~\ref{EPOCH_THM_main}, Lemma~\ref{EPOCH_LEM_Warm_up}, and a union bound, with probability at least $1-(5\frac{\pi^2}{6}+1)\delta-i\delta_1$:
\[
\LL_N=\tilde{O}(T^{2\kappa}_{\text{init}})+\mathrm{poly}(\Delta,\norm{a}^2_2,n,\beta,\kappa,\log 1/\delta,\log 1/\delta_1)\paren{\tilde{O}(N_0)+\tilde{O}(1)}
\]
To complete the bound on $\LL_N$, replace $\delta_1$ with $\delta/{i}$. Since $i$ depends logarithmically on $N$ we finally obtain that with probability at least $1-(5\frac{\pi^2}{6}+2)\delta$:
\[
\LL_N=\tilde{O}(T^{2\kappa}_{\text{init}})+\mathrm{poly}(\Delta,\norm{a}^2_2,n,\beta,\kappa,\log 1/\delta)\paren{\tilde{O}(N_0)+\tilde{O}(1)}
\]

\textbf{Martingale term bound.}
Denote $u_k\triangleq \bar{R}^{-1/2}e_k$ and $z_k\triangleq R^{1/2}\paren{\hat{y}_k-\tilde{y}_k}$. Then
 $\sum_{t=1}^{N}e^*_t\paren{\hat{y}_t-\tilde{y}_t}=\sum_{t=1}^{N}u^*_tz_t=\sum_{t=1}^{N}\sum_{i=1}^{m}u_{t,i}z_{t,i}$. 
 To apply Theorem~\ref{MART_THM_Vector} we need to slightly modify the definition of the filtration. Let $\F_{t,i}\triangleq \sigma(\F_{t}\cup\set{u_{t+1,1},\dots,u_{t+1,i}})$, with $\F_{t+1}\equiv \F_{t,m}$ and define:
 \begin{align}
 \tilde{\F}_{0}&=\F_0\\
 \tilde{\F}_{s}&=\F_{t,s-tm},\text{ if }tm+1 \le s \le (t+1)m 
 \end{align}

By  applying Theorem~\ref{MART_THM_Vector} with $\tilde{\F}_s$ we can bound the sum in terms of the square loss $\LL_N$. 
With probability at least $1-\delta$:
\begin{align}
\paren{\sum_{t=1}^{N}z^*_tz_t+1}^{-1/2}\sum_{t=1}^{N}u^*_tz_t &\le 8\log \frac{5}{\delta}+4\log \paren{\sum_{t=1}^{N}z^*_tz_t+1}
\end{align}
 or
 \begin{align}
\sum_{t=1}^{N}u^*_tz_t \le \paren{\norm{\bar{R}}_2\LL_N+1}^{1/2}\paren{8\log \frac{5}{\delta}+4\log \paren{\norm{\bar{R}}_2\LL_N+1}}
 \end{align}
where we used the fact that $z_k^*z_k=(\hat{y}_k-\tilde{y}_k)^*R(\hat{y}_k-\tilde{y}_k)\le \norm{\bar{R}}_2\norm{\hat{y}_k-\tilde{y}_k}^2_2$.

\textbf{Final step}

Finally, by a union bound, with probability at least $1-(5\frac{\pi^2}{6}+3)\delta$:
\[
\Reg_N=\tilde{O}(T^{2\kappa}_{\text{init}})+\mathrm{poly}(\Delta,\norm{a}^2_2,n,\beta,\kappa,\log 1/\delta)\paren{\tilde{O}(N_0)+\tilde{O}(1)}
\]
\hfill $\square$

\section{Stable case}\label{APP_Section_Stable}
In the case of stable systems, we can exploit the fact that the covariance matrix $\Gamma_k$ converges exponentially fast to a stead-state covariance $\Gamma_{\infty}$--see Lemma~\ref{SYS_LEM_Covariance}. 
We have the following pointwise persistency of excitation result.
Define the controllability matrix
\begin{equation}\label{STABLE_EQN_Controllability}
\C_t\triangleq\matr{{cccc}A^{t}K\bar{R}^{1/2}&\dots&AK\bar{R}^{1/2}&K\bar{R}^{1/2}}
\end{equation}
Since $(A,K)$ is controllable we also have that $\mathrm{rank}(\C_t)=n$, for $t\ge n$. As a result, the covariance matrix, is always strictly positive definite $\sigma_{\min}(\Gamma_t)\ge \sigma_{\min}(\C_n\C_n^*)>0$, for $t\ge n$--see also Lemma~\ref{SYS_LEM_Covariance}.
\begin{lemma}[Stable: pointwise persistency of excitation]\label{STABLE_LEM_PE_Pontwise}
	Consider system~\eqref{FOR_EQN_System_Innovation} with observations $y_{0},\dots,y_k$. Pick a $\tau\ge n$.
 Define
\begin{equation}\label{STABLE_EQN_Index_Functions}
\begin{aligned}
k_3(\tau,\delta)&\triangleq \tau+128\paren{m\tau^2\log 9+\tau\log2+\tau\log\tfrac{1}{\delta}}\\
k_4(k,\tau,\delta)&\triangleq \tau+\frac{64}{\min\set{8,\sigma_{\min}(\Gamma_{\tau})}}  \paren{4\tau n\log\paren{\frac{n\norm{A^{\tau}}_2^2\norm{\Gamma_{k-\tau}}_2}{\delta}+1}+8\tau \log \frac{\tau 5^m}{\delta}}.
\end{aligned}
\end{equation}
With probability at least $1-3\delta$, if \[k\ge k_3(\tau,\delta),k_4(k,\tau,\delta)\]
then
\begin{equation}\label{STABLE_EQN_PE_Pointwise}
\sum_{t=0}^k \hat{x}_t\hat{x}^*_t \succeq \frac{k-\tau+1}{4}\Gamma_{\tau}
\end{equation}
\end{lemma}
\begin{proof}
	Define $u_t\triangleq \bar{R}^{-1/2}e_k$:
	\[
	U_t\triangleq \matr{{c}u_{t-\tau}\\\vdots\\u_{t-1}}
	\]
	Observe that:
	\begin{align*}
	\hat{x}_t&=A^{\tau}\hat{x}_{t-\tau}+\C_{\tau}U_t,\\
	\Gamma_{\tau}&=\C_{\tau}\C_{\tau}^*.
	\end{align*}
	Expanding the correlations gives:
	\begin{align}
	\sum_{t=0}^k \hat{x}_t\hat{x}^*_t &\succeq 	\sum_{t=\tau}^k \hat{x}_t\hat{x}^*_t \\
	 &= A^{\tau}\sum_{t=\tau}^{k}\hat{x}_{t-\tau}\hat{x}^*_{t-\tau}(A^*)^{\tau}+A^{\tau}\sum_{t=\tau}^{k}\hat{x}_{t-\tau}U^*_{t}\C_\tau^*+\C_{\tau}\sum_{t=\tau}^{k}U_t\hat{x}^*_{t-\tau}(A^*)^{\tau}+\C_{\tau}\sum_{t=\tau}^{k}U_tU^*_{t}\C_\tau^*
	\end{align}
		The proof now is similar to Theorem~\ref{FIXED_THM_Identification} and Lemma~\ref{OutputPE_LEM_PE}. We will show that the cross terms are dominated.
	Define the primary events:
	\begin{subequations}
		\begin{alignat}{2}
		\mathcal{E}_{\bar{X}}\triangleq&\set{\sum_{t=0}^{k-\tau}\hat{x}_t\hat{x}^*_t \preceq  (k-\tau+1)\frac{n}{\delta}\Gamma_{k-\tau}}\label{STABLE_EQN_State_Upper}\\
		\mathcal{E}_E\triangleq&\set{ \sum_{t=\tau}^{k}U_kU_k^*\succeq \frac{k-\tau+1}{2}I}\\
		\mathcal{E}_{XE}\triangleq&\set{ \norm{\bar{W}_k^{-1/2}A^{\tau}\sum_{t=\tau}^{k}\hat{x}_{t}U^*_t}^2_2\le 8\tau\paren{\log \frac{\tau5^{m}}{\delta}+\frac{1}{2}\log\det\bar{W}_kW^{-1}}},\label{STABLE_EQN_Cross_State_Noise}
		\end{alignat}
	\end{subequations}
	where matrices $\bar{W}_t, W$ are:
	\begin{align}
	&\bar{W}_k\triangleq A^{\tau}\sum_{t=0}^{k-\tau}\hat{x}_t\hat{x}^*_t (A^*)^{\tau}+W,&&W\triangleq  (k-p+1)I.
	\end{align}
	The events $\mathcal{E}_{\bar{X}}, \mathcal{E}_{XE}$ occur with probability at least $1-\delta$ each--see proof of Theorem~\ref{FIXED_THM_Identification} and Theorem~\ref{MART_THM_Vector}.
By Lemma~\ref{STAT_LEM_Toeplitz}, if $k\ge k_3(\tau,\delta)$, then also $\mathcal{E}_E$ occurs with probability at least $1-\delta$.	
By a union bound $\mathcal{E}_{\bar{X}},\mathcal{E}_{E},\mathcal{E}_{XE}$ occur with probability at least $1-3\delta$.
What remains to show is that these three events imply~\eqref{STABLE_EQN_PE_Pointwise}.
The remaining proof is omitted since it is identical with the one of Lemma~\ref{OutputPE_LEM_PE}. 
\end{proof}

\subsection{Proof of Lemma~\ref{Lemma_PE_stable}}
It follows from the lemma below, which is more general. 
\begin{lemma}[Stable case: Uniform PAC bounds]\label{STABLE_UniformPE_LEM}
	Consider the conditions of Theorem~\ref{OUT_THM_Main_Bound} with $\rho(A)<1$. Select a failure probability $\delta>0$. Let $T=2^{i-1}T_{\text{init}}$ for some fixed epoch $i$ with $p=\beta \log T$ the corresponding past horizon. Consider also the definition of  $g_1(k,p,\delta)$ in~\eqref{FIXED_EQN_Cross_Terms_g}. There exists a $N_0=\mathrm{poly}(n,\beta,\kappa,\log 1/\delta,\log 1/\rho(A))$ such that with probability at least  $1-8\sum_{k=T}^{2T-1}\frac{1}{k^2}\delta$ the following events hold:
	\begin{equation}\label{STABLE_UniformPE_EQN_Bounds}
	\mathcal{E}_{\mathrm{unif}}\triangleq \set{	\begin{aligned}
		\sum_{j=p}^{k}Z_{j}Z^*_{j}&\preceq  (k-p+1)\frac{ k^2 mp}{\delta}\Gamma_{Z,k}\\
		\norm{S_k\bar{V}_k^{-1/2}}_2&\le g_1(k,p,\delta/k^2)
		\end{aligned}, \text{ for all }T\le k\le 2T-1}
	\end{equation}
	\begin{equation}\label{STABLE_UniformPE_EQN_PE}
	\mathcal{E}^{\mathrm{PE}}_{\mathrm{st,unif}}\triangleq \set{	\sum_{j=p}^{k}Z_{j}Z^*_{j}\succeq \frac{k-p+1}{32}\Gamma_{Z,k+1},\text{ for all }\max\set{N_0,T}\le k\le 2T-1}
	\end{equation}
	\end{lemma}
\begin{proof}
	Pick $\tau$ such that $\Gamma_{\tau}\succeq \frac{1}{2}\Gamma_{\infty}$. By Lemma~\ref{SYS_LEM_Covariance}, 
	\[\tau=\tilde{O}(\frac{1}{\log 1/\rho(A)}\max\set{\log\mathrm{cond}(\Gamma_{\infty}),\kappa_{\max}}).\]
	
	Next, consider also definitions of $k_1(p,\delta)$, $k_2(k,p,\delta)$ in~\eqref{FIXED_EQN_Index_Functions}, and the definitions of $k_3(\tau,\delta)$, $k_4(k,\tau,\delta)$ in~\eqref{STABLE_EQN_Index_Functions}.  Define:
	\begin{equation}\label{STABLE_UniformPE_EQN_N0}
	N_0\triangleq \min\set{t:\: \begin{aligned} k\ge k_1(\beta\log k,\delta/k^2)&, k\ge k_2(k,\beta \log k, \delta/k^2)\\k-\beta\log k \ge k_3(\tau,\delta/k^2)&, k-\beta\log k\ge k_4(k,\tau, \delta/k^2)\end{aligned},\text{ for all }k\ge t}.
	\end{equation}
Similar to the proof of Lemma~\ref{UniformPE_LEM}, by the technical Lemmas~\ref{LogIn_Lem_Log},~\ref{LogIn_Lem_LogSq} it follows that $N_0$ depends polynomially on the arguments $\beta,n,m,\log 1/\delta,\tau$.

	By the definition of $N_0$, if $k\ge N_0,T$, then:
	\[
	k\ge \max\set{k_{1}(\beta\log k,\delta/k^2),k_2(k,\beta \log k,\delta/k^2)}\ge\max\set{k_{1}(p,\delta/k^2),k_2(k,p,\delta/k^2)}.
	\]	
	Moreover,
	\[
		k-p\ge \max\set{k_{3}(\tau,\delta/k^2),k_4(k,\tau,\delta/k^2)}.
	\]
	
	Now, fix a $k$ such that $T\le k\le 2T-1$. By Theorem~\ref{FIXED_THM_Identification}, with probability at least $1-5\delta/k^2$ we have:
	\begin{align*}
	\sum_{j=p}^{k}Z_{j}Z^*_{j}&\preceq  (k-p+1)\frac{ k^2 mp}{\delta}\Gamma_{Z,k}\\
	\norm{S_k\bar{V}_k^{-1/2}}_2&\le g_1(k,p,\delta/k^2)\\
	\sum_{j=p}^{k}Z_{j}Z^*_{j}&\succeq \frac{1}{2}\O_p\bar{X}_k\bar{X}_k\O^*_p+\frac{k-p+1}{4}\Sigma_E,\text{ if }k\ge N_0
	\end{align*}
	Meanwhile by Lemma~\ref{STABLE_LEM_PE_Pontwise}, with probability at least $1-3\delta/k^2$:
	\[
	\bar{X}_k\bar{X}_k\succeq \frac{k-p-\tau+1}{4}\Gamma_{\tau}\succeq \frac{k-p-\tau+1}{8}\Gamma_{\infty}\succeq \frac{k-p+1}{16}\Gamma_{\infty},\text{ if }k\ge N_0
	\]
	where in the last inequality we used the fact that $k-p\ge k_3\ge 2\tau.$
	Combining both we have:
	\[
	\sum_{j=p}^{k}Z_{j}Z^*_{j}\succeq \frac{k-p+1}{32}\paren{\O_p\Gamma_{\infty}\O^*_p+\Sigma_E}\succeq\frac{k-p+1}{32}\Gamma_{Z,k+1},\text{ if }k\ge N_0
	\]
		
	The uniform result follows by a union bound over all $T\le k\le 2T-1$.
\end{proof}

\subsection{Proof of Theorem~\ref{STABLE_THM_Main_Bound}}
Similar to the non-explosive case, we analyze the square loss for a single epoch. 
\begin{theorem}[Square loss within epoch]\label{STABLE_EPOCH_THM_main}
	Consider the conditions of Theorem~\ref{OUT_THM_Main_Bound}. Fix two failure probabilities $\delta,\delta_1>0$  and consider $N_0$ defined as in~\eqref{STABLE_UniformPE_EQN_N0} based on $\delta$. Let $T=2^{i-1}T_{\text{init}}$ for some fixed epoch $i\ge 1$ with $p=\beta \log T$ the corresponding past horizon. Then, with probability at least $1-8\sum_{k=T}^{2T-1}\frac{1}{k^2} \delta-2\delta_1$:
	\begin{equation}\label{STABLE_EPOCH_EQN_Loss_noPE}
	\LL^{2T-1}_T\le \mathrm{poly}(n,\beta,\log 1/\delta,\log 1/\delta_1)\paren{\tilde{O}(T)+\tilde{O}(\rho(A-KC)^{2p}T^2)}.
	\end{equation}
	If moreover $T\ge N_0$ then also:
	\begin{equation}\label{STABLE_EPOCH_EQN_Loss_PE}
	\LL^{2T-1}_T\le \mathrm{poly}(n,\beta,\log 1/\delta,\log 1/\delta_1)\paren{\tilde{O}(1)+\tilde{O}(\rho(A-KC)^{2p}T)}.
	\end{equation}
\end{theorem}
\begin{proof}
\textbf{Uniform events occur with high probability. } Consider the primary events $\mathcal{E}_{\mathrm{unif}}$ and $\mathcal{E}^{\mathrm{PE}}_{\mathrm{st,unif}}$ defined in~\eqref{STABLE_UniformPE_EQN_Bounds},~\eqref{STABLE_UniformPE_EQN_PE} and:
\begin{align}
\mathcal{E}_x&=\set{\sup_{k\le 2T-1}\norm{\Gamma^{-1/2}_{k}\hat{x}_k}^2_2\le \paren{\sqrt{n}+\sqrt{2\log \frac{4T}{\delta_1}}}^2}\\
\mathcal{E}_z&=\set{\sup_{k\le 2T-1}\norm{\Gamma^{-1/2}_{Z,k}Z_k}^2_2\le \paren{\sqrt{pm}+\sqrt{2\log\frac{2T}{\delta_1}}}^2}
\end{align}
Based on Lemma~\ref{STABLE_UniformPE_LEM}, Lemma~\ref{SUP_LEM_Supremum}, and a union bound the all events $\mathcal{E}_{\mathrm{unif}}\cap \mathcal{E}^{\mathrm{PE}}_{\mathrm{st,unif}}\cap\mathcal{E}_x\cap\mathcal{E}_z$ occur with probability at least $1-8\sum_{k=T}^{2T-1}\frac{1}{k^2}\delta-\delta_1$. Now we proceed as in the proof of Theorem~\ref{EPOCH_THM_main}. We bound the square loss based on the above events.

\noindent\textbf{Bound on $\norm{S_{k-1}\bar{V}^{-1/2}_{k-1}}^2_2$. }
	From the definition of $\mathcal{E}_{\mathrm{unif}}:$
	\[
	\sup_{T\le k\le 2T-1}	\norm{S_{k-1}\bar{V}^{-1/2}_{k-1}}^2_2\le g^2(k,p,\delta/k^2)=\mathrm{poly}(n,\beta,\log 1/\delta)\tilde{O}(1)
	\]

\noindent	\textbf{Bound on $\norm{\lambda G\bar{V}^{-1/2}_t}^2_2$.}	
	We simply have: $\norm{\lambda G\bar{V}^{-1/2}_t}^2_2\le \lambda \norm{G}^2_2$
	
\noindent\textbf{Bound on $\norm{C(A-KC)^p}^2_2\norm{\bar{X}_{t-1}}^2_2$.}	
	Notice that 
	\[\norm{\bar{X}_{t-1}}^2_2\le 2T\norm{\Gamma_{2T-1}}_2\sup_{k\le 2T-1}\norm{\Gamma^{-1/2}_{k}\hat{x}_k}^2_2.\]
	Based on $\mathcal{E}_x:$
	\[
	\norm{C(A-KC)^p}^2_2\norm{\bar{X}_{t-1}}^2_2\le \mathrm{poly}(n,\log 1/\delta_1)\tilde{O}(\rho(A-KC)^{2p}T)
	\]
	
\noindent\textbf{Bound on $\norm{C(A-KC)^p}^2_{2}\sum_{k=T}^{2T-1} \norm{\hat{x}_{k-p}}^2_2$.}
It is similar to the previous step:
	\[
	\sum_{k=T}^{2T-1} \norm{\hat{x}_{k-p}}^2_2\le T \norm{\Gamma_{2T-1}}_2\sup_{k\le 2T-1}\norm{\Gamma^{-1/2}_{k}\hat{x}_k}^2_2
	\]

\noindent\textbf{Bound on the sum of $\norm{\bar{V}^{-1/2}_{k-1} Z_k}^2_2$.}
We have:
\begin{align*}
\sum_{k=T}^{2T-1}\norm{\bar{V}^{-1/2}_{k-1}Z_k}^2_2&\le \paren{\sum_{k=T}^{2T-1}\norm{\bar{V}^{-1/2}_{k-1}\Gamma^{1/2}_{Z,k}}^2_2}\sup_{k\le 2T-1}\norm{\Gamma^{-1/2}_{Z,k}Z_k}^2_2
\end{align*}
There are two cases:
\begin{align*}
\sum_{k=T}^{2T-1}\norm{\bar{V}^{-1/2}_{k-1}\Gamma^{1/2}_{Z,k}}^2_2&\le 32 \frac{2T-p-1}{T-p-1},\text{ if }T\ge N_0\\
\sum_{k=T}^{2T-1}\norm{\bar{V}^{-1/2}_{k-1}\Gamma^{1/2}_{Z,k}}^2_2&\le \frac{T}{\lambda} \norm{\Gamma_{Z,2T-1}}_2,\text{ if }T< N_0.
\end{align*}
Meanwhile, we upper-bound $\norm{\Gamma^{-1/2}_{Z,k}Z_k}^2_2$ based on $\mathcal{E}_z$.

\noindent\textbf{Final bound.} It follows from~\eqref{EPOCH_EQN_Square_Loss_Terms} and the above bounds.
\end{proof}

To prove Theorem~\ref{STABLE_THM_Main_Bound}, we now follow the same steps as in the proof of Theorem~\ref{OUT_THM_Main_Bound}, which are omitted here. It is sufficient to select
\begin{equation}\label{STABLE_EQN_beta_choice}
\beta\ge \frac{1}{\log 1/\rho(A-KC)}
\end{equation}
The final result is the following: with probability at least $1-(8\frac{\pi^2}{6}+3)\delta$:
\[
\Reg_N=\tilde{O}(T_{\text{init}})+\mathrm{poly}(n,\beta,\log 1/\delta)\paren{\tilde{O}(N_0)+\tilde{O}(1)}
\]
\hfill $\square$
\section{Alternative regret definition}\label{APP_Section_ALT}
In this section, we sketch how the online learning definition~\eqref{EXT_Online_Regret} of regret, i.e. the best linear predictor if we knew all $N$ data beforehand,  is equivalent to our definition~\eqref{FOR_EQN_Regret_Output}.

\begin{lemma}
  Consider system~\eqref{FOR_EQN_System_Innovation} with $\rho(A)\leq 1$. 
Let $y_0,\dots,y_N$ be sequence of system observations  with $\hat{y}_0,\dots,\hat{y}_N$ being the respective Kalman filter predictions. Fix a failure probability $\delta>0$. There exists a $N_0=\mathrm{poly}\paren{\log1/\delta}$ such that with probability at least $1-\delta$, if $N> N_0$ then:
	\begin{equation}\label{ALT_Result}
\sum_{k=0}^{N}\norm{y_k-\hat{y}_k}^2_2-\inf_{\mathcal{G}}\sum^{N}_{k=0}\norm{y_k-\sum^{k}_{t=1}g_{t}y_{k-t}}^2\le \mathrm{poly}(\log 1/\delta)\tilde{O}(1).
	\end{equation}
\end{lemma}
\begin{proof}
We only sketch the proof here.
We have
\begin{align}
&\inf_{\mathcal{G}}\sum^{N}_{k=0}\norm{y_k-\sum^{k}_{t=1}g_{t}y_{k-t}}^2\\
&=\inf_{\mathcal{G}}\sum_{k=0}^{N} \paren{\norm{y_k-\sum_{t=1}^{p}g_ty_{k-t}}^2+\norm{\sum_{t=p+1}^{k}g_ty_{k-t}}^2-2(y_k-\sum_{t=1}^{p}g_ty_{k-t})^*\sum_{t=p+1}^{k}g_ty_{k-t}}.
\end{align}
If we bound the magnitudes of $y_{k}$ based on Lemma~\eqref{SUP_LEM_Supremum}, then with probability at least $1-\delta$
\begin{align}
\inf_{\mathcal{G}}\sum^{N}_{k=0}\norm{y_k-\sum^{k}_{t=1}g_{t}y_{k-t}}^2&\ge \min_{g_1,\dots,g_p}\sum_{k=0}^{N} \paren{\norm{y_k-\sum_{t=1}^{p}g_ty_{k-t}}^2}-\tilde{O}(\rho^{p}\mathrm{poly}(N,\log 1/\delta)),
\end{align}
where the minimum is over all possible values for $g_i$.
Let 
\[
Z_k=\matr{{ccc}y^*_{k-p}&\cdots&y^*_{k-1}}, \text{ for }k=0,\dots, N
\]
with $y_{t}=0$ if $t<0$.
Define 
\begin{align*}
E^{+}&\triangleq \matr{{ccc}e_0&\cdots&e_{N}},\\
Y^{+}&\triangleq \matr{{ccc}y_0&\cdots&y_{N}},\\
Z&\triangleq \matr{{ccc}Z_0&\cdots&Z_{N}},\\
\end{align*}

Let $\tilde{G}=\matr{{ccc}\tilde{g}_1&\cdots&\tilde{g}_p}$ be the solution of $\min_{g_1,\dots,g_p}\sum_{k=0}^{N} \paren{\norm{y_k-\sum_{t=1}^{p}g_ty_{k-t}}^2}$. We can show that:
\[
\tilde{G}=YZ^*(ZZ^*)^{-1}.
\]
Hence,
\begin{equation}\label{ALT_EQN_AUX_1}
\min_{g_1,\dots,g_p}\sum_{k=0}^{N} \paren{\norm{y_k-\sum_{t=1}^{p}g_ty_{k-t}}^2}=\norm{Y-\tilde{G}Z}^2_F=\norm{Y-YZ^*(ZZ^*)^{-1}Z}^2_F,
\end{equation}
where $\norm{\cdot}_F$ denotes the Frobenius norm.
But we have:
\[
Y=GZ+E^++O(\rho^p\mathrm{poly}(N))
\]
Replacing $Y$ in~\eqref{ALT_EQN_AUX_1}, we obtain: 
\[
\min_{g_1,\dots,g_p}\sum_{k=0}^{N} \paren{\norm{y_k-\sum_{t=1}^{p}g_ty_{k-t}}^2}=\norm{E}^2_F-\norm{EZ^*(ZZ^*)^{-1/2}}^2_F-\tilde{O}(\rho^{p}\mathrm{poly}(N,\log 1/\delta))
\]

But the term $\norm{E}^2_F=\sum_{k=0}^{N}\norm{y_k-\hat{y}_k}^2_2$ is the Kalman filter prediction error. Using Theorem~\ref{FIXED_THM_Identification}, we can show that there exists a $N_0$ such that $ZZ^*\succeq (k-p+1)\sigma_R I$ with high probability if $N\ge N_0$. Meanwhile, using Theorem~\ref{MART_THM_Vector} combined with persistency of excitation, we can show that with probability at least $1-\delta$:
\[
\norm{EZ^*(ZZ^*)^{-1/2}}^2_F=\mathrm{poly}(\log 1/\delta)\tilde{O}(1)
\]

Combining all results:
\[
\inf_{\mathcal{G}}\sum^{N}_{k=0}\norm{y_k-\sum^{k}_{t=1}g_{t}y_{k-t}}^2\ge \sum_{k=0}^{N}\norm{y_k-\hat{y}_k}^2_2-\mathrm{poly}(\log 1/\delta)\tilde{O}(\rho^{p}\mathrm{poly}(N))
\]
Choosing $p=c\log N$ for sufficiently large $c$ gives the result.
\end{proof}
\section{Technical lemmas}\label{APP_Section_Log}
\begin{lemma}\label{LogIn_Lem_Log}
	Let $c>0$ be a positive constant. Consider the inequality:
	\[
	k\ge c \log k
	\] 
	Then, a sufficient condition for the above inequality to hold is: 
	\[
	k\ge \max\set{2c\log 2c,1}
	\]
\end{lemma}
\begin{proof}
	If $c\le e$, then the inequality is satisfied for all $k>0$. To see why this holds consider $f(k)=k-e\log k$. The minimum is attained at $f(e)=e-e\log e= 0$. Hence, $k\ge e \log k\ge c\log k$. 
	
	Next, we analyze the case $c>e$. 
We have that the function $k-c\log k$ is increasing for $k\ge c$. Moreover, $2c\log 2c \ge c$. As a result if $k\ge 2c\log 2c$ then also:
\[
k-c\log k\ge 2c\log 2c -c\log(2c\log 2c)=c\log 2c-c\log\log 2c\ge c\log 2c-\frac{c}{e}\log 2c\ge 0
\]
where we used Lemma~\ref{LogIn_Lem_LogLog}.
\end{proof}
\begin{lemma}\label{LogIn_Lem_LogSq}
	Let $c$ be a positive constant. Consider the inequality:
	\[
	k\ge c \log^2 k
	\] 
	Then, a sufficient condition for the above inequality to hold is: 
	\[
k\ge \max\set{4c\log^2 4c,4c\log 4c,1}
	\]
\end{lemma}
\begin{proof}
	If $c\le 1$, then the inequality is satisfied for $k\ge 1$. To see why this holds define $f(k)=k-\log^2{k}$. Its derivative $f'(k)=1-2\frac{\log k}{k}$ is always positive for $k\ge 1$ since from the proof of Lemma~\ref{LogIn_Lem_Log} $k\ge 2\log k$. Hence $f(k)\ge f(1)=1$.
	
	Consider now the case $c>1$ and define $g(k)=k-c\log^2 k$. Its derivative is $g'(k)=1-2c\frac{\log  k}{k}$. From Lemma~\ref{LogIn_Lem_Log} $g'(k)\ge 0$, for $k\ge \max\set{4c\log 4c,1}$.
	Now, pick $k_1=4c\log^2 4c$ and observe that $k_1\ge 4c\log 4c$ since $4c>e$ and $\log 4c>1$. Since $g$ is increasing for $k\ge k_1$, it is sufficient to prove that $g(k_1)>0$.
	We compute:
	\[
	c\log^2(k_1)=c\paren{\log 4c+\log\log 4c}^2\stackrel{(i)}{\le}c\paren{\log 4c+\frac{1}{e}\log 4c}^2\le 4c\log^2 4c=k_1,
	\]
	where $(i)$ follows from Lemma~\ref{LogIn_Lem_LogLog} below.
\end{proof}
\begin{lemma}\label{LogIn_Lem_LogLog}
Let $c\ge e$, then the following inequality holds:
\[
\log\log c \le \frac{1}{e}\log c\]	
\end{lemma}
\begin{proof}
	Consider function $f(c)=\frac{1}{e}\log c-\log\log c$ and compute the derivative:
	\[
	f'(c)=\frac{1}{ec}-\frac{1}{c\log c}
	\]
	The minimum is attained at $e^e$. Hence
	\[
	f(c)\ge f(e^e)=0
	\]
	for all $c\ge e$.
\end{proof}